%% file: paper.tex
\documentclass{article}

\usepackage{microtype}
\usepackage{graphicx}
\usepackage{hyperref}
\PassOptionsToPackage{numbers, sort&compress}{natbib}

\usepackage[sglblindworkshop]{neurips_2026}

\usepackage{mathtools}
\usepackage{amsthm}
\usepackage{thmtools,thm-restate} 
\theoremstyle{plain}
\newtheorem{theorem}{Theorem}[section]

\theoremstyle{definition}
\newtheorem{definition}[theorem]{Definition}
\newtheorem{example}[theorem]{Example}
\newtheorem{assumption}{Assumption}

\theoremstyle{remark}

\usepackage[textsize=tiny]{todonotes}

\usepackage{latexsym}
\usepackage[linesnumbered, ruled, vlined]{algorithm2e}
\usepackage{amsmath,amssymb}
\usepackage{xspace}
\usepackage{booktabs}
\usepackage[most]{tcolorbox}
\tcbuselibrary{skins} 
\usepackage{pgfplotstable, pgfmath}
\usepackage[noenumitem,notheorems]{eda}
\usepackage{bbm}
\usepackage{colortbl}
\usepackage{prettyplots}
\usepackage[utf8]{inputenc} 
\usepackage[T1]{fontenc}    
\usepackage{hyperref}       
\usepackage{url}            
\usepackage{booktabs}       
\usepackage{amsfonts}       
\usepackage{nicefrac}       
\usepackage{microtype}      
\usepackage{xcolor}         
\usepackage{eda}
\pgfplotsset{
    colormap/viridis,
}
\graphicspath{ {./figs/} }
\pgfplotsset{compat=1.18}
\usetikzlibrary{calc} 
\usetikzlibrary{positioning}
\usetikzlibrary{backgrounds}
\usetikzlibrary{external}
\usetikzlibrary{intersections} 
\usepgfplotslibrary{fillbetween}
\usepgfplotslibrary{colorbrewer} 
\usetikzlibrary{pgfplots.colormaps} 
\usetikzlibrary{pgfplots.statistics} 
\usetikzlibrary{pgfplots.groupplots} 
\usepgfplotslibrary{ternary}
\usetikzlibrary{fit}
\usetikzlibrary{arrows}
\usepackage{tikzviolinplots} 
\usepackage{pgfplotstable} 
\usepackage{scontents} 
\usepackage{array}   
\usepackage{numprint}  
\usepackage{ifthen}
\usepackage{framed} 
\usepackage{enumitem} 
\usepackage{wrapfig}
\usepackage{listings}
\usepackage{rotating}
\usepackage{color}
\usepackage{subcaption} 
\usepackage{graphicx}
\usepgfplotslibrary{fillbetween}

\usepackage{adjustbox}
\usepackage{booktabs}
\usepackage{multirow}

\tcbset{
    plotbox/.style={
        enhanced, 
        colback=gray!3,
        colframe=gray!25,
        arc=3pt,              
        boxrule=0.6pt,
        left=1pt, right=1pt, top=8pt, bottom=4pt,
        fonttitle=\tiny\bfseries,
        coltitle=gray!70,
        attach boxed title to top center={yshift=-2mm},
        boxsep=2pt,
            equal height group=plots,
            valign=center 
        boxed title style={
            colback=white, 
            colframe=gray!25, 
        }
    }
}

\input{macro/defines}

\title{Identifying Structural Biases from\\Causal Mechanism Shifts}
\author{%
  Praharsh Nanavati\thanks{Corresponding Author e-mail: \texttt{praharsh.nanavati@cispa.de}} \qquad \qquad 
  Jilles Vreeken \qquad \qquad
  David Kaltenpoth \\
  CISPA Helmholtz Center for Information Security \\
}
\setcitestyle{numbers,square,comma}
\begin{document}

\maketitle

\begin{abstract}
Causal discovery methods commonly assume that all data is independently and identically distributed (i.i.d.) and that there are no unmeasured variables affecting the system. In practice, these assumptions are often violated, leading to inaccurate inference. In this paper, we study how to identify hidden confounding and selection biases from causal mechanism shifts. In particular, we show that structural biases lead to dependent mechanism shifts. That is, by considering for which variables the mechanisms change given data from different environments, we can tell which variables are unbiased, which are subject to hidden confounding, and which are undergoing selection bias. We formalize this into an empirically testable criterion based on mutual information, and show under which conditions it identifies structural biases. To tell which nodes are subject to what kind of bias, we introduce the \ourmethod algorithm. Experiments on synthetic and real-world data show that \ourmethod works well in practice, accurately recovering affected variable sets and types of biases, outperforming the state-of-the-art by a wide margin.
Our implementation is available at\quad \url{https://github.com/niftynans/strubi}.

\end{abstract}

\input{sections/introduction.tex}
\input{sections/related_work.tex}

\input{sections/preliminaries.tex}

\input{sections/theory.tex}
\input{sections/algorithm.tex}

\input{sections/experiments.tex}

\input{sections/discussion_conclusion.tex}


\clearpage
\bibliographystyle{plainnat}
\bibliography{bib/abbreviations,bib/bib-jilles,bib/bib-paper}
\clearpage

\appendix
\input{sections/appendix.tex}

\clearpage

\end{document}

%% file: macro/defines.tex
\DeclareMathOperator{\pval}{p-val}

\newcommand{\figbf}[1]{\textit{#1}}

\newcommand{\ourmethod}{\textsc{StruBI}\xspace}
\newcommand{\topic}{\textsc{Topic}\xspace}

\newcommand{\Pit}{\Pi^{\ast}}
\newcommand{\Piti}{\Pit_i}
\newcommand{\Pio}{\Pi^{\mathrm{obs}}}
\newcommand{\Pioi}{\Pio_i}
 
\newcommand{\X}{X}

\newcommand{\Pc}{P^{c}}
\newcommand{\Pci}{P^{c}_{i}}

\newcommand{\Pco}{P^c_{\mathrm{obs}}}
\newcommand{\Pcoi}{P^c_{i, \mathrm{obs}}}

\newcommand{\Gc}{\mathcal{G}}
\newcommand{\Gt}{\Gc^\ast}
\newcommand{\Cc}{\mathcal{C}}
\newcommand{\Eb}{\mathbb{E}}  

\newcommand{\Acc}{A_{c,c'}}

\DeclareMathOperator{\anc}{anc}
\DeclareMathOperator{\ch}{ch}
\DeclareMathOperator{\cov}{cov}

\DeclareMathOperator{\pa}{pa}
\DeclareMathOperator{\mi}{I}

\newcommand{\pat}{\pa^{\ast}}
\newcommand{\pati}{\pa^{\ast}_i}
\newcommand{\paoi}{\pa_i}

\newcommand{\ContingAB}{\mathcal M}

\newcommand{\f}[1]{f^{(#1)}}

\pgfplotsset{
	abstraction line/.style={
		dotted,
		thick
	}
}

\definecolor{clust1col}{RGB}{101,66,134}
\definecolor{clust2col}{HTML}{0891b2}
\definecolor{nipsvio}{RGB}{101,66,134}
\definecolor{testblue}{HTML}{06b6d4}
\tikzset{
	baseline_mechanism/.style={fill=steelblue!20, opacity=0.3},	 
	intervention_node/.style={cover_node,minimum height = 0.35cm, minimum width = 12pt,yshift = 5pt,xshift = -2mm},
	intervention1/.style={intervention_node,fill=goldenrod!45, 
	},
	clustering1/.style={intervention_node,fill=clust1col!45, 
	},
	clustering2/.style={intervention_node,fill=clust2col!45, 
	},
	intervention2/.style={intervention_node,fill=steelblue!40, 
	},
	intervention3/.style={intervention_node,fill=nipsvio!40,
	},
	intervention4/.style={intervention_node,fill=dollarbill!50, draw=dollarbill},
	intervention5/.style={intervention_node,fill=nipsvio!40, draw=nipsvio},
	intervention0/.style={intervention_node,fill=green(ryb)!45, draw=green(ryb)},
	itemset mark/.style={baseline=-4ex},
	text itemset/.style={draw=orange},
	pattern_node/.style={rectangle,rounded corners,thick, align=center, font=\small, inner sep=1.5pt},
	cover_node/.style={rectangle,line width=.5pt, rounded corners=0.1pt, align=right, font=\small, inner sep=1.5pt},
	select_node/.style={circle,thick, align=center, font=\small},
	ynode/.style={rectangle, draw=black, minimum width = 15pt, minimum height= 15pt},
	xnode/.style={rectangle, draw=black, minimum width = 15pt, minimum height= 15pt, text height = 6.5pt},
	label/.style={},
	dag_node/.style={circle,
		align=center, font=\scriptsize,
		text=black,
		line width = .4pt,
		scale = 1.0,text width=.5cm},
	invis_node/.style={draw=white!0},	i_node/.style={ 
		align=center, font=\scriptsize, 
		line width = .4pt,
		scale = 0.8,text height=.25cm, text width=.35cm},	i2/.style={i_node, goldenrod},
	i3/.style={i_node,steelblue},
	i1/.style={i_node, nipsvio} , 
	i4/.style={i_node, green(ryb)},
	invis_node/.style={draw=white!0}
}

\definecolor{s0}{HTML}{ff9b02} 
\definecolor{s1}{HTML}{F9D923} 
\definecolor{s2}{HTML}{134e6f} 
\definecolor{s3}{HTML}{ff6150} 
\definecolor{s4}{HTML}{1ac0c6} 
\definecolor{s5}{HTML}{36AE7C} 

\definecolor{e0}{HTML}{ff6150}
\definecolor{e1}{HTML}{1ac0c6}
\definecolor{e2}{HTML}{F9D923}
\definecolor{e3}{HTML}{36AE7C}
\definecolor{e4}{HTML}{134e6f}
\definecolor{e5}{HTML}{ff6150}
\definecolor{e6}{HTML}{1ac0c6}
\definecolor{e7}{HTML}{ff9b02}

\definecolor{abst1}{HTML}{E4C1F9}
\definecolor{abst2}{HTML}{4F7CAC}
\definecolor{abst3}{HTML}{9EEFE5}
\definecolor{abst4}{HTML}{541388}

\SetKwComment{tcpas}{\{}{\}}
\SetCommentSty{textnormal}
\SetArgSty{textnormal}
\SetKwRepeat{Do}{do}{while}
\SetKw{False}{false}
\SetKw{True}{true}
\SetKw{Null}{null}
\SetKwInOut{Output}{output}
\SetKwInOut{Input}{input}
\SetKw{AND}{and}
\SetKw{OR}{or}
\SetKw{Continue}{continue}

\pgfdeclarelayer{background}
\pgfdeclarelayer{foreground}
\pgfsetlayers{background,main,foreground}  

\definecolor{mygreen}{rgb}{0.0, 0.5, 0.0}
\definecolor{mygray}{rgb}{0.3, 0.3, 0.3}

\definecolor{variolila}{RGB}{115,100,137}
\definecolor{lilacgray}{RGB}{152,150,164}
\definecolor{red1}{RGB}{201,59,69}
\definecolor{cc1}{rgb}{0.59, 0.78, 0.64}
\definecolor{niceblue}{RGB}{3, 79, 132}
\definecolor{iceblue}{RGB}{128, 182, 207}
\definecolor{cc1}{rgb}{0.59, 0.78, 0.64}
\definecolor{niceblue}{RGB}{3, 79, 132}
\definecolor{iceblue}{RGB}{128, 182, 207}
\definecolor{ca3}{rgb}{0.83, 0.69, 0.22}
\definecolor{red1}{RGB}{201,59,69}
\definecolor{red2}{RGB}{169,22,48}
\definecolor{red3}{RGB}{249,129,115}
\definecolor{cl1}{rgb}{0.61, 0.77, 0.89} 
\definecolor{cl2}{rgb}{0.64, 0.68, 0.82} 
\definecolor{cl3}{rgb}{0.85, 0.65, 0.13} 
\definecolor{ca1}{rgb}{0.6, 0.81, 0.93} 
\definecolor{ca2}{RGB}{139, 204, 109} 
\definecolor{rosequartz}{RGB}{247,202,201}
\definecolor{serenity}{RGB}{145,168,209}
\definecolor{peachecho}{RGB}{247,120,107}
\definecolor{snorkelblue}{RGB}{3,79,132}
\definecolor{buttercup}{RGB}{250,224,60}
\definecolor{limpetshell}{RGB}{152,221,222}
\definecolor{lilacgray}{RGB}{152,150,164}
\definecolor{fiesta}{RGB}{221,65,50}
\definecolor{icedcoffee}{RGB}{177,143,106}
\definecolor{greenflash}{RGB}{121,199,83}
\definecolor{purple}{RGB}{153,0,153}
\definecolor{turquise}{RGB}{0,153,153}
\definecolor{poop}{RGB}{203,178,52}
\definecolor{blue1}{RGB}{215, 216, 233}
\definecolor{blue2}{RGB}{243, 243, 248}
\definecolor{varioblue}{RGB}{68, 114, 157}
\definecolor{variogreen}{RGB}{125, 171, 113}

%% file: sections/introduction.tex
\section{Introduction}
Methods for causal inference generally require causal sufficiency and i.i.d.-ness \citep{spirtes2000causation,pearl:09:causalitybook}. That is, to provably recover the causal ground truth, they require that all causally relevant variables are measured (sufficiency) and that the samples in the data are all independent and identically distributed (i.i.d.). While convenient in theory, neither of these are likely to hold in practice \citep{mameche2024identifying,mameche:23:linc}. Often, we do not know what the causally relevant variables are, and even if we do, we cannot always measure them. Often, we also do not know whether a system has a representative sample, or whether it is subject to selection bias. In either case, this can lead to causal models that can be arbitrarily wrong because the data will show spurious dependencies \citep{mameche2024identifying, spirtes2013causal, dai2025selection}. For example, if `shark attacks' and `ice cream sales' are statistically dependent, it may lead us to believe there is a causal dependency, even if they are actually conditionally independent given their \emph{unmeasured} confounder `season'. Similarly, if we only study the survival rates of fighter jets that returned from combat, then we get a distorted view of where more protective armor is most needed. This raises the question, can we detect whether the data at hand is unbiased, subject to hidden confounding, or subject to selection bias? 

In this paper, we show that by considering non-i.i.d. data collected from different contexts, the effects of confounding and selection bias are identifiable by considering distribution shifts of the observed variables. In particular, we show that actual, \emph{intrinsic}, mechanism shifts leave measurably different \emph{observed} mechanism shifts between the three cases. We give an example in Fig.~\ref{fig:illu}. Here, $X$ and $Y$ are children of observed variable $W$ and are either unbiased (left), confounded (center), or undergoing selection bias (right) through a latent collider $V$. We consider data gathered in five contexts, $c_1,\ldots,c_5$. In each context, interventions affect specific variables (hammers) and change their underlying generating process (colored boxes).  

In the unbiased case, where we measured all relevant variables, the intervention on $V$ in context $c_1$ (purple) do not affect the observed mechanisms of the other variables. The conditionals $P_{X \mid V, W }$ and $P_{Y \mid V, W}$, and the marginal $P_{W}$ all remain the same, they change independently in other contexts (yellow, blue, green). Under confounding (center), however, an intrinsic mechanism change of $V$ lead to observed changes in its children ($X$ and $Y$). Under selection bias (right), however, an intrinsic mechanism shift in $V$ leads to observed shifts in all of its upstream variables ($X$, $Y$, and $W$). 

We formally show under which conditions we can leverage this observation to identify and classify between the three cases. To do so, we assume that mechanisms shifts are sparse and independent. Importantly, unlike existing work~\cite{dai2025selection,dai2025latent}, our approach is non-parametric, making it generally applicable. To put our theory to practice, we propose the~\ourmethod algorithm for structural bias identification from empirical data.~\ourmethod operates on a given causal graph $\mathcal{G}$ and employs the KCI test by~\cite{zhang2012kernel} to determine if mechanisms change between contexts. We use adjusted mutual information to determine whether the mechanism shifts of different variables behave independently or not, and to classify between confounding or selection, consider many of the upstream variables change together. 

We empirically evaluate~\ourmethod on a wide range of synthetic data, comparing it to five state of the art methods. The experiments show that~\ourmethod outperforms its competitors by a wide margin, and that using an inferred causal graph gives comparable results to using the ground truth causal network. We verify that~\ourmethod is applicable on real-world data through a case study on cell-signaling data, where it correctly classifies between confounding and selection bias, and unlike its competitors, retrieves high-quality sets of affected variables in both settings. 

All code is available in the supplementary. We postpone all proofs to the appendix. 

\input{figs/tex_figs/illustration_new.tex}

%% file: figs/tex_figs/illustration_new.tex
\usetikzlibrary{fit}
\newcommand{\figureHeight}{8}

\newcommand{\contextSeq}{$c_{1}$/0, $c_2$/1, $c_3$/2, $c_4$/3, $c_5$/4} 
\newcommand{\xSeq}{ $ X_{1}$/0, $X_2$/1, $X_3$/2, $X_4$/3, $X_5$/4} 
\newcommand{\drawseq}[3]{
	\foreach \symbol/\i in \contextSeq{
		\node (lower\i) [anchor = north west, minimum height = 0.35cm, minimum width = 12pt, lilacgray] at (\i*0.45+#1,#2 + 0.08)  {};
		\node (left\i) [anchor = north west, minimum height = 0.35cm, minimum width = 12pt, lilacgray] at (\i*0.45+#1-0.6,#2 + 0.08+0.5)  {};
		\node (upper\i) [anchor = north west, minimum height = 0.35cm, lilacgray, minimum width = 12pt,cover_node, baseline_mechanism, ] at (\i*0.45+#1,#2 + 0.08+0.4) {\scriptsize #3};
	}
}

\newcommand{\drawseqAnnotated}[4]{
\foreach \symbol/\i in #4{
	\node (lower\i) [anchor = north west, minimum height = 0.35cm, minimum width = 12pt] at (\i*0.45+#1,#2 + 0.08) {\scriptsize {\symbol \strut}};
	\node (left\i) [anchor = north west, minimum height = 0.35cm, minimum width = 12pt, lilacgray] at (\i*0.45+#1-0.6,#2 + 0.08+0.5)  {};
	\node (upper\i) [anchor = north west, minimum height = 0.35cm, lilacgray, minimum width = 12pt,cover_node, baseline_mechanism, ] at (\i*0.45+#1,#2 + 0.08+0.4) {\scriptsize #3};
}
} 
\newcommand{\drawSeq}[3]{\drawseq{#1}{#2}{$p_#3$}}
\newcommand{\drawSeqAnnotated}[4]{\drawseqAnnotated{#1}{#2}{$p_#3$}{#4}}

\setlength{\intextsep}{0pt}

\tcbset{
        mybox/.style={
            enhanced,
            colback=gray!2,
            colframe=gray!20,
            arc=2pt,
            boxrule=0.5pt,
            left=8pt, right=8pt, top=8pt, bottom=8pt,
            boxsep=0pt,
			height=5.5cm
        }
    }

\begin{figure*}[t] 
    \centering
    \scalebox{0.75}{
    \begin{tabular}{@{} c c c @{}}
        
        \begin{tcolorbox}[mybox, width=5.2cm]
            \centering
            \begin{tikzpicture}[outer sep=0pt, inner sep=0pt]
                \tikzset{every node/.style={anchor=north west}}
                
                \def\posX{0}; \def\fH{8}; 
                \def\up{1}; \def\upmid{2.5}; \def\offG{1.5}; \def\offN{1}; 
                \def\upN{0.5}; \def\upL{.6}; \def\offP{0.5}; \def\offC{1.6}; \def\upC{0.4};
                    
                \node (causal) at (\posX,\fH) {\small \underline{\textit{Unbiased:}}};
                
                \node at (\posX, \fH - \up) {\small $G^\star:$}; 
                \node (x) [dag_node] at (\posX + \offG, \fH - \up - \upL) { $X$}; 
                \node (v) [dag_node] at (\posX + \offG + 0.9, \fH - \up - \upN + 0.75) {$V$};
                \node (y) [dag_node] at (\posX + \offG + \offN, \fH - \up - \upL) { $Y$};
                \node (w) [dag_node] at (\posX + \offG + 0.10, \fH - \up - \upL + 0.85) { $W$}; 

                \node (i1) [i1] at (\posX + \offG + \offN + 0.1, \fH - \up - \upN + 1.15) {$c_1$ \includegraphics[width=0.5cm]{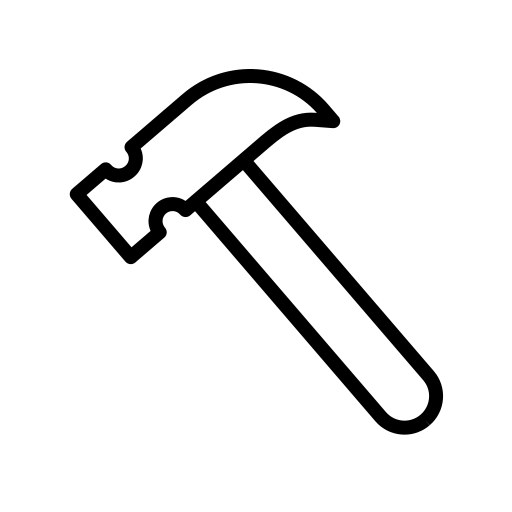}};
                \node (i2) [i2] at (\posX + \offG - 0.25, \fH - \up - \upL + 0.45) {$c_2$ \includegraphics[width=0.5cm]{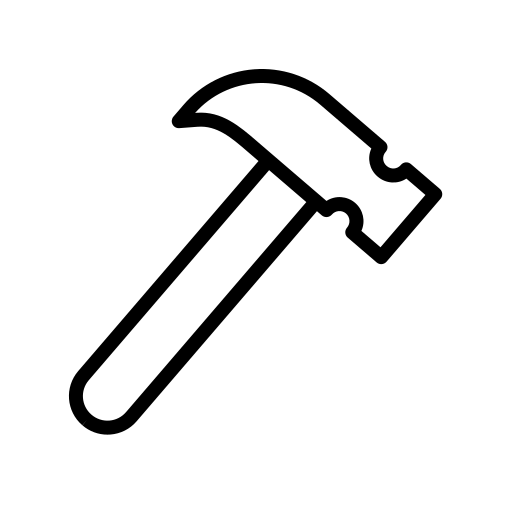}};
                \node (i3) [i3] at (\posX + \offG + \offN + 0.2, \fH - \up - \upL + 0.45) {$c_5$ \includegraphics[width=0.5cm]{figs/hammer-rev.png}};
                \node (i4) [i4] at (\posX + \offG - 0.25, \fH - \up - \upL + 1.25) {$c_4$ \includegraphics[width=0.5cm]{figs/hammer.png}};
                
                \path[-{latex[]}] (w) edge[] (x) (w) edge[] (y) (v) edge[] (x) (v) edge[] (y);
                
                \tikzset{labelnode/.style={minimum width=1.5cm, inner sep=0pt, font=\small}}
                \node[labelnode] at (\posX, \fH - \upmid) {$P_V:$};
                \node[labelnode] at (\posX, \fH - \upmid - \offP) {$P_W:$};
                \node[labelnode] at (\posX, \fH - \upmid - 2*\offP) {$P_{X \mid V, W }:$};
                \node[labelnode] at (\posX, \fH - \upmid - 3*\offP) {$P_{Y \mid V, W }:$};
                
                \drawSeq{\posX + \offC}{\fH - \upmid - \upC}{v};
                \node [nipsvio, intervention3] at (upper0) {\scriptsize $\tilde p_v$};
                \drawSeq{\posX + \offC}{\fH - \upmid - \offP - \upC}{w};
                \node [green(ryb), intervention0] at (upper3) {\scriptsize $\tilde p_w$};
                \drawSeq{\posX + \offC}{\fH - \upmid - 2*\offP - \upC}{x};
                \node [goldenrod, intervention1] at (upper1) {\scriptsize $\tilde p_x$};
                \drawSeqAnnotated{\posX + \offC}{\fH - \upmid - 3*\offP - \upC}{y}{\contextSeq};  
                \node [steelblue, intervention2] at (upper4) {\scriptsize $\tilde p_y$};  
            \end{tikzpicture}
        \end{tcolorbox}
        &
        \begin{tcolorbox}[mybox, width=5.2cm]
            \centering
            \begin{tikzpicture}[outer sep=0pt, inner sep=0pt]
                \tikzset{every node/.style={anchor=north west}}
                
                \def\posX{0}; \def\fH{8}; 
                \def\up{1}; \def\upmid{2.5}; \def\offG{1.5}; \def\offN{1}; 
                \def\upN{0.5}; \def\upL{.6}; \def\offP{0.5}; \def\offC{1.6}; \def\upC{0.4};

                \node (confounded) at (\posX,\fH) {\small \underline{\textit{Hidden Confounding:}}};
                
                \node at (\posX, \fH - \up) {\small $G^\star:$}; 
                \node (x) [dag_node] at (\posX + \offG, \fH - \up - \upL) { $X$}; 
                \node (v) [dag_node] at (\posX + \offG + 0.9, \fH - \up - \upN + 0.75) {$V$};
                \node [draw, dotted, thick, inner sep=-3pt, fit=(v), circle] {};
                \node (y) [dag_node] at (\posX + \offG + \offN, \fH - \up - \upL) { $Y$};
                \node (w) [dag_node] at (\posX + \offG + 0.1, \fH - \up - \upL + 0.85) { $W$};

                \node (i1) [i1] at (\posX + \offG + \offN + 0.30, \fH - \up - \upN + 1.15) {$c_1$ \includegraphics[width=0.5cm]{figs/hammer-rev.png}};
                \node (i2) [i2] at (\posX + \offG - 0.25, \fH - \up - \upL + 0.45) {$c_2$ \includegraphics[width=0.5cm]{figs/hammer.png}};
                \node (i3) [i3] at (\posX + \offG + \offN + 0.2, \fH - \up - \upL + 0.45) {$c_5$ \includegraphics[width=0.5cm]{figs/hammer-rev.png}};
                \node (i4) [i4] at (\posX + \offG - 0.25, \fH - \up - \upL + 1.25) {$c_4$ \includegraphics[width=0.5cm]{figs/hammer.png}};

                \path[-{latex[]}] (v) edge[dotted] (x) (v) edge[dotted] (y) (w) edge[] (x) (w) edge[] (y);

                \tikzset{labelnode/.style={minimum width=1.5cm, inner sep=0pt, font=\small}}
                \node[labelnode] at (\posX, \fH - \upmid - \offP) {$P_W:$};
                \node[labelnode] at (\posX, \fH - \upmid - 2*\offP) {$P_{X\mid W}:$};
                \node[labelnode] at (\posX, \fH - \upmid - 3*\offP) {$P_{Y\mid W}:$};
                
                \drawSeq{\posX + \offC}{\fH - \upmid - \offP - \upC}{w};
                \node [green(ryb), intervention0] at (upper3) {\scriptsize $\tilde p_w$};
                \drawSeq{\posX + \offC}{\fH - \upmid - 2*\offP - \upC}{x};
                \node [nipsvio, intervention3] at (upper0) {\scriptsize $\tilde p_v$};
                \node [goldenrod, intervention1] at (upper1) {\scriptsize $\tilde p_x$};
                \drawSeqAnnotated{\posX + \offC}{\fH - \upmid - 3*\offP - \upC}{y}{\contextSeq};  
                \node [nipsvio, intervention3] at (upper0) {\scriptsize $\tilde p_y$};
                \node [steelblue, intervention2] at (upper4) {\scriptsize $\tilde p_y$};  
            \end{tikzpicture}
        \end{tcolorbox}
        &
        \begin{tcolorbox}[mybox, width=5.2cm]
            \centering
            \begin{tikzpicture}[outer sep=0pt, inner sep=0pt]
                \tikzset{every node/.style={anchor=north west}}
                
                \def\posX{0}; \def\fH{8}; 
                \def\up{1}; \def\upmid{2.5}; \def\offG{1.5}; \def\offN{1}; 
                \def\upN{0.5}; \def\upL{.6}; \def\offP{0.5}; \def\offC{1.6}; \def\upC{0.4};

                \node (confounded2) at (\posX,\fH) {\small \underline{\textit{Selection Bias:}}};
                
                \node at (\posX, \fH - \up) {\small $G^\star:$}; 
                \node (x) [dag_node] at (\posX + \offG, \fH - \up - \upL) { $X$}; 
                \node (v) [dag_node] at (\posX + \offG + 0.9, \fH - \up - \upL + 0.95) {$V$};
                \node [draw, dotted, thick, inner sep=-1pt, fit=(v)] {};
                \node (y) [dag_node] at (\posX + \offG + \offN, \fH - \up - \upL) { $Y$};
                \node (w) [dag_node] at (\posX + \offG + 0.1, \fH - \up - \upL + 0.85) { $W$};

                \node (i1) [i1] at (\posX + \offG + \offN + 0.30, \fH - \up - \upN + 1.35) {$c_1$ \includegraphics[width=0.5cm]{figs/hammer-rev.png}};       
                \node (i2) [i2] at (\posX + \offG - 0.25, \fH - \up - \upL + 0.45)  {$c_2$ \includegraphics[width=0.5cm]{figs/hammer.png}};
                \node (i3) [i3] at (\posX + \offG + \offN + 0.2, \fH - \up - \upL + 0.45) {$c_5$ \includegraphics[width=0.5cm]{figs/hammer-rev.png}};
                \node (i4) [i4] at (\posX + \offG - 0.25, \fH - \up - \upL + 1.25) {$c_4$ \includegraphics[width=0.5cm]{figs/hammer.png}};

                \path[-{latex[]}] (x) edge[dotted] (v) (y) edge[dotted] (v) (w) edge[] (x) (w) edge[] (y);

                \tikzset{labelnode/.style={minimum width=1.5cm, inner sep=0pt, font=\small}}
                \node[labelnode] at (\posX, \fH - \upmid - \offP) {$P_W:$};
                \node[labelnode] at (\posX, \fH - \upmid - 2*\offP) {$P_{X\mid W}:$};
                \node[labelnode] at (\posX, \fH - \upmid - 3*\offP) {$P_{Y\mid W}:$};

                \drawSeq{\posX + \offC}{\fH - \upmid - \offP - \upC}{w};
                \node [nipsvio, intervention3] at (upper0) {\scriptsize $\tilde p_w$};
                \node [green(ryb), intervention0] at (upper3) {\scriptsize $\tilde p_w$};
                \drawSeq{\posX + \offC}{\fH - \upmid - 2*\offP - \upC}{x};
                \node [nipsvio, intervention3] at (upper0) {\scriptsize $\tilde p_x$};
                \node [goldenrod, intervention1] at (upper1) {\scriptsize $\tilde p_x$};
                \drawSeqAnnotated{\posX + \offC}{\fH - \upmid - 3*\offP - \upC}{y}{\contextSeq};  
                \node [nipsvio, intervention3] at (upper0) {\scriptsize $\tilde p_y$};
                \node [steelblue, intervention2] at (upper4) {\scriptsize $\tilde p_y$};

            \end{tikzpicture}
        \end{tcolorbox}
    \end{tabular}
    }
\caption{\figbf{Structural biases introduce dependent mechanism shifts.} We consider three systems of four variables across five different contexts. 
                The unbiased system (left) exhibits independent mechanism shifts. 
                Under confounding (center), an \emph{actual} mechanism change on \emph{hidden} confounder $V$ in context $c_1$, leads to \emph{observed} mechanism changes in variables $X|V,W$ and $Y|V,W$ for the same context. 
                Under selection on a \emph{hidden} collider (right), a mechanism change on that collider ($V$) leads to \emph{observed} mechanism changes for \emph{all} its ancestors, $X$, $Y$, and even $W$.} \label{fig:illu} 


\end{figure*}

%% file: sections/related_work.tex
\section{Related Work}
Inferring causality from observational data is a fundamental challenge in statistics \citep{rubin1974estimating,spirtes2000causation,pearl:09:causalitybook}. In a nutshell, it is only possible to do so under assumptions of the data generating model~\cite{pearl:09:causalitybook}. Standard assumptions include the causal Markov condition, causal faithfulness, causal sufficiency, and i.i.d.-ness. Under these assumptions, classical constraint-based \citep{spirtes2000causation,zhang2008completeness} and score-based \citep{chickering2002learning,ramsey2017million} methods can recover causal structures up to Markov equivalence.

Causal discovery under latent confounding is an active topic of research. The FCI family~\citep{spirtes2000causation,colombo2012learning,ogarrio2016hybrid} can discover and mark edges that are likely confounded. Nested Markov Models (NMMs) \citep{shpitser2014introduction,shpitser2018acyclic,richardson2017nested,evans2019smooth} leverage Verma constraints to identify latent factors. DCD~\citep{bhattacharya2021differentiable} combines NMMs with differentiable constraints \citep{zheng2018dags} to discover partial DAGs. \citet{reddy:22:counterfactual} utilize mutual information on observed distributions, and \citet{karlsson:23:confounding} address latent confounding for effect estimation under a fixed graph structure. \citet{kaltenpoth2023causal, kaltenpoth2023nonlinear} model latent confounders directly, allowing identification of which nodes \emph{share} latent confounders.  

There exists much less research on causal discovery under selection. FCI~\citep{spirtes2000causation} is sound under selection, but \citet{spirtes2013causal} show that single observational studies are generally insufficient to distinguish between confounding or selection. To resolve this, \citet{versteeg2022local} identify local patterns to detect causal effects under both biases. Recent work has expanded these capabilities to temporal dependencies in sequential data \citep{zheng2024detecting}, "twin graph" frameworks for interventional settings \citep{dai2025selection}, and rank-based constraints for linear Gaussian models \citep{dai2025latent}. These methods respective require it i.i.d. or sequential data, or strict parametric assumptions. 

Recent research leverages non-i.i.d. data, in particular distribution shifts across different environments, to relax the assumptions necessary to identify causal DAGs. Existing methods primarily focus on handling latent confounding. JCI~\citep{mooij2020joint} extends existing causal discovery methods with a context variable.  \citet{perry:sms} focus on sparse mechanism shifts, whereas \textsc{CD-NOD}~\cite{huang2020causal} and \textsc{CoCo}~\citep{mameche2024identifying} both exploit independent mechanism changes. Our approach leverages multi-environment data to identify the presence of, and, which nodes are affected by hidden confounding resp. selection bias by detecting dependent mechanism shifts.

%% file: sections/preliminaries.tex
\section{Preliminaries}
\label{sec:setup}
In this section, we establish conditions necessary to distinguish confounding from selection bias in non-i.i.d. settings. We begin by defining context-indexed observed distributions and show how they are generated by latent causal mechanisms. 
We then ask the central question: what patterns in the observed context shifts differentially reveal latent confounding and respectively selection?

\subsection{Problem Setting}

We consider variables $V = X \cup Z \cup S$, where $X$, $Z$, and $S$ denote the sets of observed variables, hidden confounders, and hidden selection variables, respectively. We assume that $X = (X_1\ldots,X_d)$ consists of $d$ dimensions and the system is observed across multiple contexts $c \in \Cc$, with $n_c = |\Cc|$.
For each context, the full data-generating distribution is $P^{c}(V) = P(V \mid C=c)$. We assume that all these distributions are \emph{Markov} and \emph{faithful} with respect to the same causal DAG $\Gt = (V, E)$, with edge $(i,j) \in E$ whenever $V_i$ is a direct cause of $V_j$. We write $\pati = \{V_j : (j,i) \in E\}$ for the parents of $V_i$ in $\Gt$, and $\paoi = \pati \cap X$ for observed parents of variables $X_i$. We further write $\ch_i$ for the set of all children of $V_i$, $\anc_i$ for the set of all ancestors, and $\anc^+(U) = \anc(U) \cup U$.

We assume that latent confounders $Z$ are exogenous, while latent selection variables $S_{\ell}$ are sink nodes in $\Gt$ with $\pat(S_{\ell}) \subseteq X$. The observed distribution in context $c$ is then given by
\[
  \Pco(X) := \Pc(X \mid R(S)=1)\,,
\]
where $R = R(S) = 1$ indicates retention of a sample. Absent selection, $R \equiv 1$ and $\Pco(X) = \Pc(X)$. We make the standard assumption \citep{kaltenpoth2023identifying} that selection is \emph{positive and nondegenerate}, i.e., $0 < \Pc(R(S)=1 \mid \pa(R))$ for all $c$, and $\Pc(R(S) \mid \pa(R)) < 1$ for some $c \in \Cc$. This excludes selection rules where selection is so strict that there is no data left in some contexts.

We do not assume causal sufficiency over $X$ alone. Instead, all hidden common causes of observed variables are assumed to be included in $Z$, and all selection-induced conditioning is represented by $S$. 

The question is therefore: what can observed mechanisms $\Pco$ tell us about the graph $\Gt$ over $V$?
\paragraph{Problem Statement}
\textit{Given the set of observed distributions $\{\Pco(X)\}_{c \in \Cc}$, our goal is to identify subsets of $X$ whose observed local mechanisms shift jointly across contexts, and to distinguish, when possible, whether these shifts are better explained by latent confounding or by selection bias.}

\subsection{Intrinsic and Observed Causal Mechanism Shifts}
A key distinction in our setting is between \emph{intrinsic} mechanism shifts in all variables, $V_i \in V$ and \emph{observed} mechanism shifts after marginalizing latent confounders and conditioning on selection.

For any variable $V_i \in V$, we refer to the conditional distribution $P^c(V_i \mid \pati)$ as its intrinsic local mechanism in context $c$.
This induces an intrinsic partition $\Piti$ of the context set: two contexts $c, c'$ belong to the same block when the mechanisms generating $V_i$ are the same, \mbox{$\Pc(V_i \mid \pati) = P^{c'}(V_i \mid \pati)$}. 
We denote the block to which $c$ belongs by $\Piti(c)$ and therefore have $\Piti(c) = \Piti(c')$ in these cases.
We call a pair of contexts $c, c'$ such that $\Piti(c) \neq \Piti(c')$ an \emph{intrinsic shift}.
We assume that these intrinsic shifts are jointly independent and sparse~\citep{mameche2024identifying,perry:sms}.
\begin{assumption}[Sparse Independent Intrinsic Shifts]
  We assume that intrinsic partitions $\{\Piti\}_{i}$ are jointly independent, and that for any given $V_i$ shifts are sparse, $0 < P(\Piti(C) \neq \Piti(C')) < 0.5$.
\end{assumption}
This assumption commonly holds in realistic settings where shifts occur due to interventions on a system \citep{perry:sms}. For example, in biological data, data from observed distributions is much easier to obtain than interventional data, and such interventions usually target only one feature at a time.

Just as the intrinsic shifts of $\Pci = \Pc(V_i \mid \pati)$ induce partitions $\Piti$, the observed distributions $\Pcoi = \Pc(X_i \mid \paoi, R = 1)$ can also shift and induce observed partition $\Pioi$ of the contexts. In general, every intrinsic shift of $X_i$ between $c, c'$ also implies an observed shift between $c, c'$.
For two contexts $c,c'$, we define the \emph{observed affected set}
\[
A_{c,c'} := \{X_i \in X : \Pio(c) \neq \Pio(c') \}\,.
\]
Our goal in this paper is to relate the observed partitions $\Pio$ of variables $X$ to the intrinsic partitions $\Pit$ of all variables $V$.
In the next section we show that variables $X_i, X_j$ which are members of $\Acc$ for many $c, c'$ are prime targets for suspecting the influence of latent mechanism shifts.

%% file: sections/theory.tex
\section{Theory}
\label{sec:theory}


In this section we connect how observed partitions relate to the effects of latent confounding and selection by studying the network topology of nodes whose observed partitions correlate.
Proofs for all results can be found in the appendix.

\subsection{Difference in Shifts in Observed Variables Based on the Structural Bias}
We begin with a simple example to illustrate how latent parents of a variable $X$ and latent selection on a variable $X$ can have different impacts on its upstream variables.
\begin{example}
  \label{ex:confounding-vs-selection}
  Suppose that variable $X$ is caused by a latent $Z$ and observed $W$, and that $Z, W, \epsilon_X$ are jointly independent,
  \begin{equation}
    X = \alpha Z + \gamma W + \epsilon_X, \qquad \mu_c = E[Z \mid C = c]
  \end{equation}
  If $\mu_c=0$ and $\mu_{c'}=\delta$, then the observed conditional mean changes as
  \begin{align}
    \mathbb{E}[X \mid W=w, C=c]
    &= \gamma w, \\
    \mathbb{E}[X \mid W=w, C=c']
    &= \alpha \delta + \gamma w.
  \end{align}
  Thus the observed distribution $\Pc(X \mid W)$ changes across contexts because it marginalizes over the shifted latent parent $Z$. In contrast, the observed distribution $\Pc(W)$ does not shift.

  Now suppose that there is no latent confounder, but that selection acts through a descendant $S$ of $X$,
  \begin{equation}
    X = \gamma W + \epsilon_X,
    \qquad
    S = aX + \epsilon_S,
    \qquad
    R = \mathbf{1}\{S > \tau_c\}\,,
  \end{equation}
  where $R=1$ denotes that a sample is retained in context $c$.
  By Bayes' theorem, we get
  \begin{align}
    p_c(x \mid w, R=1)
    \propto
    \frac{\Pc(R=1 \mid X=x, W=w)}{\Pc(R=1 \mid W=w)}
    =
    \frac{\Pc(\epsilon_S > \tau_c - ax)}
    {\int \Pc(\epsilon_S > \tau_c - au) p(u \mid w)\,du}\,.
  \end{align}
  When $a \neq 0$, this weight depends on $x$, so the observed distribution $\Pc(X \mid W, R=1)$ changes when $\tau_c$ changes. Furthermore, the observed distribution $\Pc(W \mid R = 1)$ also changes:
  \begin{align}
    p_c(w \mid R=1)
    \propto \frac{\int \Pc(\epsilon_S > \tau_c - au) p(u \mid w)\,du\,.}{\Pc(R=1)}\,.
  \end{align}
\end{example}
Hence selection makes both its parent $X$ and its observed ancestor $W$ shift, while a changing latent parent affects only the observed distributions of $X$ that directly marginalize over it.


Next we show that this crucial asymmetry between latent confounding and selection holds generically for large classes of parametric systems.
\begin{restatable}[Observed Local Signatures of Confounding and Selection]{theorem}{theoreticalidentifiability}
  \label{thm:generic-shift-faithfulness}
  Let \(\Gt\) be the true graph over $V$, with precisely one latent variable, $\left| Z \cup S \right| = 1$.
  Let $A$ be given by
  \[
    A = \begin{cases}
      \ch_{\Gt}(Z), & \text{for a latent confounder},\\
      \anc_{\Gt}(S), & \text{for selection}\,.
        \end{cases}
  \]
  Let \(\{P_\theta^c: c \in C,\theta\in\Theta\}\) be a parametric family with \(\Theta\subseteq\mathbb{R}^d\) open and connected, such that all joint, marginal and conditional densities are real analytic in $\theta$ on a non-empty common support.
Let \(c\) and \(c'\) be two contexts in which only the latent variable's mechanism changes. Suppose that for every \(X_i\in A\), there exists at least one $\theta_0 \in \Theta$ such that $\Pco(X_i\mid \pa_{\mathcal{G}_X}(X_i); \theta_0) \neq P^{c'}_{\mathrm{obs}}(X_i \mid \pa_{\mathcal{G}_X}(X_i); \theta_0)$.

   Then $A_{c,c'} = A$ for all $\theta$ outside a Lebesgue-null subset of $\Theta$.
\end{restatable}
The conditions of this theorem generically hold for large classes of parametric families: exponential families, linear Gaussian SEMs, parametric additive noise models, generalized additive models, splines, neural networks with many common activation functions, normalizing flows, and many more.

In the following we therefore assume throughout that latent mechanism shifts induce the correct observed mechanism shifts.


\subsection{Ancestral Coverage and Minimal Explanations}
The previous theorem motivates using the graphical topology of sets of variables with correlated observed partitions. While there is no necessary topological relationship between variables whose mechanism shifts are correlated by latent confounding, in selection bias the set of variables whose partitions are correlated contains for any contained variable also all its ancestors.

\begin{definition}[Ancestral Coverage]
  Let $A \subseteq V$ be a subset of the nodes of the DAG $\mathcal{G}=(V,E)$. We define the \emph{ancestral coverage} of the subset $A$ as
  \[
    \cov_{\mathcal{G}}(A) := \frac{\bigl|A \bigr|} {\bigl|\anc_{\mathcal{G}}(A)\bigr|}\,.
  \]
  We call $A$ \emph{ancestrally closed} if $\cov_{\mathcal{G}_X}(A)=1$, i.e., $A = \anc(A)$.
\end{definition}
By Theorem~\ref{thm:generic-shift-faithfulness}, under sparse mechanism shifts, coverage $\cov(\Acc) \approx 1$ in the case of latent selection, but $\cov(\Acc) \ll 1$ in the case of latent confounding. 
\begin{restatable}[Ancestral Coverage of Confounding and Selection Effects]{proposition}{expectedcoverage}
  \label{lem:expected_coverage}
  Let $\Gt$ be a directed Erd\H{o}s-R\'enyi graph with independent edge probabilities $p$. Assume that the recovered set $\Acc$ is generated either by a single latent confounder or by a single latent selection variable.
  \begin{enumerate}
  \item If $\Acc$ is created by selection, then $\cov(\Acc) = 1$.
  \item If $\Acc$ is created by a confounder, then $\cov(\Acc) < 1$ with probability $1 - O(e^{-pk(d-k)/2})$.
  \end{enumerate}
\end{restatable}
Of course, it is always possible to construct adversarial examples in which the latent confounder $Z$ targets an ancestrally closed set. However, such explanations are not minimal in the following sense.
\begin{restatable}[Minimal Latent Variable Sets]{proposition}{lvminimality}\label{cor:lv_minimality}
  Let $\Acc \subseteq X$ be nonempty and ancestrally closed in $\Gt$. Assume that there exists at least one parent-child pair $X_i, X_j \in \Acc$. Then among all single-latent-node explanations of the correlated set $\Acc$, the edge-minimal explanation is that of a single selector $S$ with sink nodes of $\Acc$ as parents, $\pa(S) = \left\{ X_i \in \Acc : \ch(X_i) \cap X = \emptyset \right\}$.
\end{restatable}
Based on these results, we next introduce a practical algorithm which leverages the notion of ancestral coverage to distinguish confounding from selection bias in the finite data setting.


%% file: sections/algorithm.tex
\section{The \ourmethod Algorithm}

\begin{algorithm}[t]
	\caption{Structural Bias Identification -- $\ourmethod (\Gc)$}
    \label{alg:structural_bias}
    \Input{ Observed data over $\X, C$; thresholds $\tau_1, \tau_2$; observed causal graph $\Gc$.}
    \Output{ Subsets of $\X$ subject to selection bias respectively confounding.}  
    \BlankLine
    
    \ForEach{variable $X_i \in \X$}{
        Compute $p_{c,c'} = \pval \left( \Pc(X_i \mid \pa_i) \neq P^{c'}(X_i \mid \pa_i) \right)$ for all pairs $(c, c')$\; \label{line:pval}
        Convert $\{ p_{c,c'} \} $ to a partition $\Pi_i$ (grouping contexts by mechanism invariance)\; \label{line:partition}
    }
    
    \BlankLine
    Construct an affinity matrix $\ContingAB$ where $\ContingAB_{ij} = \text{AMI}(\Pi_i, \Pi_j)$\; \label{line:affinity}
    Find the set of connected components $\mathcal{K} = \{K_1, \dots, K_m\}$ in the graph induced by $\ContingAB$\; \label{line:components}
    Filter connected components of $\mathcal{K}$ such that $X_S = \{K_i \in \mathcal{K} : |K_i| > 1\}$\; \label{line:filtered_comp}
    \BlankLine
    \For{each subset $S \in X_S$}{
    $\bar{m}(S) \leftarrow \text{mean}(\{\text{AMI}(\Pi_i, \Pi_j) \mid i, j \in X_S, i < j\})$\; \label{line:internal_sync}
    \eIf{$\bar{m}(S) < \tau_1$}{
        Label $S$ as \textbf{Unbiased}\; \label{line:unbiased}
    }{
        \eIf{$\operatorname{cov}_{\mathcal{G}}(S) \geq \tau_2$}{
            Label $S$ as \textbf{Selection Bias}\; \label{line:sel}
        }{
            Label $S$ as \textbf{Confounding}\; \label{line:conf}
        }
    }
}
\Return{Set of subsets $X_S$ with assigned labels}\;
\end{algorithm}

In this section, we introduce the \ourmethod algorithm and give the pseudocode in Algorithm \ref{alg:structural_bias}. The algorithm takes as input the observed data over contexts, the observed causal graph $\Gc$, and thresholds $\tau_1$ and $\tau_2$. 
$\tau_1$ is a threshold for the mean of the pairwise Adjusted Mutual Information (AMI) values $\bar{m}(\cdot)$ for each detected subset, $S$ to detect the presence of a latent variable. We show how we calculate AMI values in Appx. \ref{sec:appx_alg}.
$\tau_2$ is a threshold for $\operatorname{cov}_{\mathcal{G}}$ to distinguish selection bias from hidden confounding. 
We arrive at the thresholds $0.1$ and $0.6$ for $\tau_1$ and $\tau_2$ respectively, as shown in Appx. Sec. \ref{sec:samp_comp}.

\paragraph{Detecting Causal Mechanism Shifts}
We perform a conditional independence test to detect mechanism changes, for each causal mechanism of an observed variable $X_i$ and each pair of environments (line \ref{line:pval}) as done by \citet{mameche2024identifying},  
resulting in the following p-values,
\begin{equation}
    p_{c,c'} = \text{p-val} \left( P^c(X_i \mid \text{pa}_i) \neq P^{c'}(X_i \mid \text{pa}_i) \right).
\end{equation}
We perform this independence test using the KCI test by \citet{zhang2012kernel} for all variables that have parents in $\mathcal{G}$. For marginal tests, we perform the Kolmogorov-smirnov test \citep{massey1951kolmogorov}.
Based on $p_{c,c'}$, we obtain partitions,$ \left\{ \Pi_i^{\text{obs}} \right\}_{i \in \{1, \dots, |V|\}}$, as defined in Section \ref{sec:setup}, that describe which observed variables have significant overlaps in their causal mechanism shifts (l.~\ref{line:partition}). 

\paragraph{Detecting presence of latent variables}
We take inspiration from \citet{mameche2024identifying} and construct an affinity matrix $\mathcal{M}$ using AMI scores between observed partitions as edge weights (l.  \ref{line:affinity}).
We extract non-trivial connected components from the resulting weighted graph by considering subsets with $\bar{m}(S) > \tau_1$, to identify node subsets jointly influenced by the same latent variable (l.  \ref{line:components}-\ref{line:filtered_comp}).
However, we use binary indicators of distribution shifts to build our partitions as opposed to \citet{mameche2024identifying} who cluster p-values to derive partitions and then detect the presence of a latent variable by doing spectral clustering on the affinity matrix obtained from the pairwise AMI scores instead. 

\paragraph{Identifying Structural Biases}
For each remaining subset, we check if checking if $\operatorname{cov}_{\mathcal{G}} < \tau_2$ and predict confounding if true, else selection bias (l. \ref{line:sel}-\ref{line:conf}).

%% file: sections/experiments.tex
\section{Experiments}
\label{sec:experiments}
In this section, we empirically evaluate \ourmethod on synthetic and real-world data. We focus on classifying between unbiasedness, confounding, and selection, and on how well the affected sets of nodes are recovered. We present additional results, including runtime analysis, in the Appendix.

\paragraph{\ourmethod.} We consider three main versions of \ourmethod. First, we consider \ourmethod\!($\Pi^\text{obs}$), which is given the \emph{true} partitions $\Pi^\text{obs}$ of the observed variables. Second, we consider \ourmethod\!($\mathcal{G^*}$), which is given the \emph{true} observed causal structure $\mathcal{G^*}$. 
Third, we evaluate \ourmethod\!($\hat{\mathcal{G}}$), which is given a \emph{learned} causal graph $\hat{\mathcal{G}}$ obtained using \topic~\citep{xu2025information}. 

In addition, to assess the robustness of \ourmethod against graph mis-specification, we also consider \ourmethod\!($\tilde{\mathcal{G}}$), which operates on a perturbed graph $\tilde{\mathcal{G}}$ generated by randomly flipping edges of the ground truth DAG. We provide details and the results for this variant in Appx.~\ref{sec:perturbed}.


\input{figs/tex_figs/main_ablation_new.tex}

\paragraph{Baselines} We compare \ourmethod against five strong baselines. 
FCI~\cite{spirtes:99:fci} can perform causal discovery in the presence of hidden confounding and selection bias. We consider FCI applied on data pooled across all contexts (FCI-P), as well as JCI-FCI~\cite{mooij2020joint}, which extends FCI to multi-context data. 
We also compare against CoCo~\cite{mameche2024identifying} as a method that can determine hidden confounding from mechanism changes. We include Latent-Selection~\citep{dai2025latent} which focuses on latent variable causal discovery under selection bias. We apply it both on pooled data (LS-P) as well as independently on each context and aggregating the results (LS-C).
For more details, see Appx. \ref{sec:ls_deets}.

\input{figs/tex_figs/baselines_new.tex}

\paragraph{Synthetic Data}
To evaluate with known ground truth, we consider synthetic data. We adopt the experimental setup of \citet{huang2020causal} and \citet{mameche2024identifying}. We generate synthetic data using an Erd\H{o}s-R\'enyi model over $m \in \{3, 5, 7, 9, 11, 13\}$ observed variables. The number of contexts $m_c$ is scaled as $2m$. 
We sample from the following structural equation model (SEM)
\begin{equation}
\label{eq:sem_eq}
    X_i^{(c)} = \sum_{j \in \mathrm{pa}(X_i)} w_{ij} \cdot f\!\left(X_j^{(c)}\right) 
    + \mu_i^{(c)} + N_i^{(c)} \; ,
\end{equation}
where the edge weights $w_{ij}$ are sampled from $\mathcal{U}(0.7, 1.3)$ for observed 
interactions and $\mathcal{U}(1.7, 2.3)$ for interactions involving latent variables. 
The functional forms $f$ are drawn uniformly at random from $\{x, x^2, x^3, \tanh, \text{sinc}\}$ 
and remain invariant across contexts. Mechanism shifts are introduced by re-sampling the 
bias term $\mu_i^{(c)}$ from a jittered discrete set $\{0, 5, 10\} \pm \mathcal{U}(-0.5, 0.5)$. 
The exogenous noise $N_i^{(c)}$ is drawn with equal probability from either a Gaussian 
distribution $\mathcal{N}(0, \sigma^2)$ or a Uniform distribution $\mathcal{U}(-\sigma\sqrt{3}, \sigma\sqrt{3})$ 
with $\sigma = 0.05$.

Each latent confounder $Z_i$ is modeled as a source node with edges to a random subset of 
at least two variables from a pool of nodes in the graph avoiding explicit ancestral closure. 
Selection bias is introduced via hidden colliders $S_i$, 
whose parents are a randomly selected subset of observed variables. The final dataset 
is filtered by retaining only samples where $S_i$ exceeds the $s^{\text{th}}$ percentile 
of its marginal distribution. We set $s = 85$ for our synthetic data, creating a strong selection bias. 

All data is normalized globally to zero mean and unit variance. 
We provide further details, along with an ablation study, and an analysis of runtimes in Appx.~\ref{sec:extra_details}.

\subsection{Detecting Structural Biases}
First, we are interested in understanding whether and how well \ourmethod can detect structural biases. To this end, we generate synthetic data as above, and consider \ourmethod\!($\Pi^\text{obs}$), \ourmethod\!($\mathcal{G^*}$), and \ourmethod\!($\tilde{\mathcal{G}}$). We give the results in Fig.~\ref{fig:oracles}. On the left-hand side, we report the $F1$ scores for classifying the type of structural bias (i.e. none, confounding, selection) resp. for determining which nodes are affected. \ourmethod\!($\Pi^\text{obs}$) shows our theory works; knowing the actual mechanism shifts, \ourmethod performs perfectly. More interestingly, \ourmethod given an inferred graph $\hat{\mathcal{G}}$ performs essentially as well as when given the true graph $\mathcal{G}^*$. The right-hand panel shows the AMI and coverage values. We see that classification performance goes hand in hand with how well partitions are recovered (top row), as well as that when more nodes share a confounder, it becomes harder to classify it from the case of selection with many ancestral nodes. 

Next, we evaluate how well \ourmethod performs compared to the baselines. To this end we consider \ourmethod given graph $\hat{\mathcal{G}}$ inferred using \topic~\cite{xu2025information}, JCI-FCI and FCI-P given the causal skeletons, CoCo given the ground truth graph $\mathcal{G^*}$, and LS-P and LS-C. 
We are interested in (i) how well can each identify the presence of latent variables, (ii) identify the affected variables, and (iii) how well can they identify the type of structural bias. We run each method on the same data as above, and present the results in Fig.~\ref{fig:baselines}. We see that for each of the three tasks, \ourmethod outperforms all baselines by a large margin. LS-P and LS-C do not do very well, probably because of their parametric and structural assumptions and model misspecification. CoCo is the best performing competitor, and the only one that improves in subset recovery as the DAG size increases. \ourmethod also outperforms the baselines in terms of the accuracy, precision, and recall. We give those results in Appx. Fig.~\ref{fig:baseline_metrics}.

\subsection{Case Study -- Cell Signaling}
To see how well \ourmethod performs in a real-world settings, we consider the (\textit{infamous}) \citet{sachs2005causal} flow cytometry dataset. The data contains samples of eleven protein and phospholipid components in human immune cells that were studied under different molecular interventions. We show the consensus causal graph for the Sachs dataset in Fig.~\ref{fig:sachs-scenarios}a. When we run \ourmethod, we find dependencies in the mechanism shifts of proteins Akt and Mek, PKA and PKC, and Plcg and P38. This aligns with documented signaling crosstalk and shared regulatory architectures \citep{mendoza2011ras, nishizuka1995protein, canovas2021diversity}. We give more details in Section \ref{sec:extra_sachs}.

\input{figs/tex_figs/sachs.tex}

For the main experiment, we focus on the PKC node, as it possesses a sufficient number of descendants and a complex upstream topology beyond its immediate parents. Following the design of~\citet{mameche:23:linc}, we induce hidden confounding by dropping variable PKC from the dataset, and selection bias by first filtering the data to retain only samples where PKC exceeds the 80th percentile of its marginal distribution, and then dropping it. We consider \ourmethod and CoCo given the consensus causal network. The other baselines do not accept a causal graph as input, so we let those infer it themselves. We run each method on the unbiased data, data where we induce confounding, and on data where we induce selection. We present the results of \ourmethod in Fig.~\ref{fig:sachs-scenarios} and of the baselines in Appx.~Fig.~\ref{fig:sachs-baselines}. 


\ourmethod correctly classifies confounding and selection bias in their respective cases. In the confounded setting, \ourmethod correctly infers three confounded edges (Jnk, PKA, Raf), reports two edges to causal descendants (Mek, AKT) that according to \emph{another} consensus graph by~\citet{meinshausen2016methods} are in fact correct, and misses only one edge (P38). In the case of selection bias, it correctly retrieves six out of eight ancestors, additionally recovering edges from Jnk, P38, and AKT that are correct compared to the graph by~\citet{meinshausen2016methods} that includes cycles are at least partly correct. Of the competitors, CoCo performs well on the confounded setting, but fails to recover accurate subsets in the case of selection bias. JCI-FCI recovers in both cases only two indirectly related nodes (Erk and Akt). LS-C in both cases incorrectly returns that all nodes in the graph are affected.  

In summary, \ourmethod performs well, even on real-world data on which our assumptions are unlikely to hold. It correctly determines unbiasedness, confounding, and selection bias, and recovers many of the truly affected nodes.



%% file: figs/tex_figs/main_ablation_new.tex
\begin{figure*}[t] 
    \centering
    \renewcommand{\prtickfontsize}{\tiny}
    \renewcommand{\prlabelfontsize}{\tiny}
    \renewcommand{\prlegendfontsize}{\tiny}

    \tcbset{
        mybox/.style={
            enhanced,
            colback=gray!2,
            colframe=gray!20,
            arc=2pt,
            boxrule=0.5pt,
            left=1pt, right=1pt, top=4pt, bottom=4pt,
            boxsep=0pt
        }
    }

    \pgfplotsset{
        base style/.style={
            pretty line,
            scale only axis,
            width=2.0cm,
            height=1.2cm,
            unbounded coords=discard,
            xmin=3, xmax=13, 
            enlargelimits=false,
            xtick={3, 5, 7, 9, 11, 13},
            xlabel style={font=\prlabelfontsize, at={(ticklabel cs:0.5)}, anchor=north, yshift=-2pt},
            y label style={font=\prlabelfontsize, at={(ticklabel cs:0.5)}, anchor=south, yshift=-2pt},
            tick label style={font=\prtickfontsize},
            clip=true,
            restrict x to domain=3:13,
            mark size=0.7pt,
            title style={font=\tiny, yshift=-2pt}, 
        },
        no x/.style={xlabel={}, xticklabels={}},
        no y/.style={ylabel={}, yticklabels={}}
    }

    \newcommand{\addoracleplot}[5]{
        \addplot[#1, name path=U, draw=none, forget plot] table[x=X, y expr=\thisrow{#3}+\thisrow{#3_std}, col sep=comma] {#4};
        \addplot[#1, name path=D, draw=none, forget plot] table[x=X, y expr=\thisrow{#3}-\thisrow{#3_std}, col sep=comma] {#4};
        \addplot[#1, fill opacity=0.1, forget plot] fill between [of=U and D];
        \addplot[thick, #1, mark=#2] table[x=X, y=#3, col sep=comma] {#4};
        \ifx\&#5\&\else\addlegendentry{#5}\fi
    }

    \def\colPartitions{pr-color-purple} 
    \def\colTOPIC{pr-color1a}       
    \def\colFull{pr-color3a}      

    \renewcommand{\arraystretch}{0}
    \setlength{\tabcolsep}{0.1pt}
    \begin{tabular}{@{} p{4.2cm} p{10cm} @{}}
        
        \begin{tcolorbox}[mybox, width=3.8cm]
            \centering
            \begin{tabular}{c}
                \begin{tikzpicture}
                    \begin{axis}[base style, no x,  title={Overall}, ylabel={Class. $F_1$}, ymin=0, ymax=1.1,
                        legend columns=-1, legend to name=LegOracles]
                        \def\f{figs/csv_for_figs/main_ablations/main_f1.csv}
                        \addoracleplot{\colPartitions}{square*}{partitions}{\f}{\ourmethod\!($\Pi^\text{obs}$)}
                        \addoracleplot{\colFull}{triangle*}{full}{\f}{\ourmethod\!($\mathcal{G}^*$)}
                        \addoracleplot{\colTOPIC}{*}{TOPIC}{\f}{\ourmethod\!($\hat{\mathcal{G}}$)}
                    \end{axis}
                \end{tikzpicture} \\
                \vspace{0.3cm} \\
                \begin{tikzpicture}
                    \begin{axis}[base style, ylabel={Subset $F_1$}, ymin=0, ymax=1.1, xlabel={Dag Size $\lvert\mathcal{G}\rvert$}]
                        \def\f{figs/csv_for_figs/main_ablations/main_subset.csv}
                        \addoracleplot{\colPartitions}{square*}{partitions}{\f}{}
                        \addoracleplot{\colFull}{triangle*}{full}{\f}{}
                        \addoracleplot{\colTOPIC}{*}{TOPIC}{\f}{}
                    \end{axis}
                \end{tikzpicture}
            \end{tabular}
        \end{tcolorbox}
        &
        \begin{tcolorbox}[mybox, width=10.5cm]
            \centering
            \begin{tabular}{c c c}
                \begin{tikzpicture}
                    \begin{axis}[base style, no x, title={Unbiased}, ylabel={AMI}, ymin=0, ymax=1.1]
                        \def\f{figs/csv_for_figs/main_ablations/main_ami.csv}
                        \addoracleplot{\colPartitions}{square*}{none_partitions}{\f}{}
                        \addoracleplot{\colFull}{triangle*}{none_full}{\f}{}
                        \addoracleplot{\colTOPIC}{*}{none_TOPIC}{\f}{}
                    \end{axis}
                \end{tikzpicture} &
                \begin{tikzpicture}
                    \begin{axis}[base style, no x, no y, title={Confounding}, ymin=0, ymax=1.1]
                        \def\f{figs/csv_for_figs/main_ablations/main_ami.csv}
                        \addoracleplot{\colPartitions}{square*}{confounder_partitions}{\f}{}
                        \addoracleplot{\colFull}{triangle*}{confounder_full}{\f}{}
                        \addoracleplot{\colTOPIC}{*}{confounder_TOPIC}{\f}{}
                    \end{axis}
                \end{tikzpicture} &
                \begin{tikzpicture}
                    \begin{axis}[base style, no x, no y, title={Selection Bias}, ymin=0, ymax=1.1]
                        \def\f{figs/csv_for_figs/main_ablations/main_ami.csv}
                        \addoracleplot{\colPartitions}{square*}{collider_partitions}{\f}{}
                        \addoracleplot{\colFull}{triangle*}{collider_full}{\f}{}
                        \addoracleplot{\colTOPIC}{*}{collider_TOPIC}{\f}{}
                    \end{axis}
                \end{tikzpicture} \\
                \vspace{0.3cm} \\
                \begin{tikzpicture}
                    \begin{axis}[base style, ylabel={Coverage}, ymin=0, ymax=1.1, xlabel={Dag Size $\lvert\mathcal{G}\rvert$}]
                        \def\f{figs/csv_for_figs/main_ablations/main_coverage.csv}
                        \addoracleplot{\colPartitions}{square*}{none_partitions}{\f}{}
                        \addoracleplot{\colFull}{triangle*}{none_full}{\f}{}
                        \addoracleplot{\colTOPIC}{*}{none_TOPIC}{\f}{}
                    \end{axis}
                \end{tikzpicture} &
                \begin{tikzpicture}
                    \begin{axis}[base style, no y, ymin=0, ymax=1.1, xlabel={Dag Size $\lvert\mathcal{G}\rvert$}]
                        \def\f{figs/csv_for_figs/main_ablations/main_coverage.csv}
                        \addoracleplot{\colPartitions}{square*}{confounder_partitions}{\f}{}
                        \addoracleplot{\colFull}{triangle*}{confounder_full}{\f}{}
                        \addoracleplot{\colTOPIC}{*}{confounder_TOPIC}{\f}{}
                    \end{axis}
                \end{tikzpicture} &
                \begin{tikzpicture}
                    \begin{axis}[base style, no y, ymin=0, ymax=1.1, xlabel={Dag Size $\lvert\mathcal{G}\rvert$}]
                        \def\f{figs/csv_for_figs/main_ablations/main_coverage.csv}
                        \addoracleplot{\colPartitions}{square*}{collider_partitions}{\f}{}
                        \addoracleplot{\colFull}{triangle*}{collider_full}{\f}{}
                        \addoracleplot{\colTOPIC}{*}{collider_TOPIC}{\f}{}
                    \end{axis}
                \end{tikzpicture}
            \end{tabular}
        \end{tcolorbox}
    \end{tabular}

    \vspace{1.0em}
    \centerline{\ref{LegOracles}}

    \caption{\figbf{Structural Bias Detection with \ourmethod.} Given are $F1$-scores for bias classification and subset recovery (left) and the mean pairwise AMI and Coverage values for the unbiased, confounding, and the selection bias cases (right). The values for the AMI and Coverage are low for unbiased subsets, low and high respectively for confounded cases, and high for both for the selection cases. 
    }
    \label{fig:oracles}
\end{figure*}
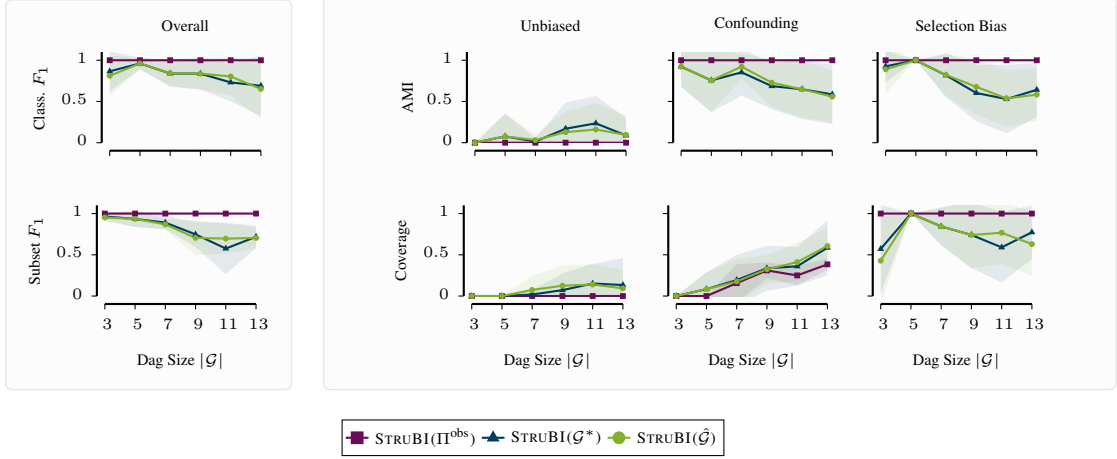

%% file: figs/tex_figs/baselines_new.tex
\begin{figure*}[t] 
    \centering
    \renewcommand{\prtickfontsize}{\tiny}
    \renewcommand{\prlabelfontsize}{\tiny}
    \renewcommand{\prlegendfontsize}{\tiny}

    \pgfplotsset{
        every axis/.append style={
            pretty line,
            scale only axis,       
            width=0.6\linewidth,  
            height=1.4cm,
            unbounded coords=discard,
            xmin=3, xmax=13, 
            enlargelimits=false,
            xtick={3, 5, 7, 9, 11, 13},
            xlabel={Dag Size $|\mathcal{G}|$},
            xlabel style={at={(ticklabel cs:0.5)}, anchor=north, yshift=-4pt},
            y label style={font=\prlabelfontsize},
            x label style={font=\prlabelfontsize},
            clip=true,
            restrict x to domain=3:15,
            mark size=0.7pt,       
            title={}, 
        }
    }

    \newcommand{\addBaselineEnvelope}[1]{
        \addplot[gray, name path=EnvU, draw=none, forget plot] table[x=X, 
            y expr={max(\thisrow{coco}+\thisrow{coco_std}, 
                   max(\thisrow{fci-jci}+\thisrow{fci-jci_std}, 
                   max(\thisrow{fci-pooled}+\thisrow{fci-pooled_std}, 
                   max(\thisrow{dai_pooled}+\thisrow{dai_pooled_std}, 
                       \thisrow{dai_context}+\thisrow{dai_context_std}))))}, 
            col sep=comma] {#1};
        \addplot[gray, name path=EnvD, draw=none, forget plot] table[x=X, 
            y expr={min(\thisrow{coco}-\thisrow{coco_std}, 
                   min(\thisrow{fci-jci}-\thisrow{fci-jci_std}, 
                   min(\thisrow{fci-pooled}-\thisrow{fci-pooled_std}, 
                   min(\thisrow{dai_pooled}-\thisrow{dai_pooled_std}, 
                       \thisrow{dai_context}-\thisrow{dai_context_std}))))}, 
            col sep=comma] {#1};
        \addplot[gray!30, fill opacity=0.4, forget plot] fill between [of=EnvU and EnvD];
    }

    \newcommand{\addbaselineplot}[6]{
        \addplot[#1, name path=U, draw=none, forget plot] table[x=X, y expr=\thisrow{#3}+\thisrow{#3_std}, col sep=comma] {#4};
        \addplot[#1, name path=D, draw=none, forget plot] table[x=X, y expr=\thisrow{#3}-\thisrow{#3_std}, col sep=comma] {#4};
        \addplot[#1, fill opacity=#6, forget plot] fill between [of=U and D];
        \addplot[thick, #1, mark=#2] table[x=X, y=#3, col sep=comma] {#4};
        \ifx\&#5\&\else\addlegendentry{#5}\fi
    }

    \def\colFull{pr-color1a}       
    \def\colCoCo{pr-color1c}       
    \def\colFCIJ{pr-color2a}       
    \def\colFCIP{pr-color1b}       
    \def\colLSP{pr-color3a}        
    \def\colLSC{pr-color-gray4}    

    \begin{minipage}{0.32\textwidth}
        \centering
        \begin{tikzpicture}
            \begin{axis}[title = {\tiny Latent Discovery}, ylabel={$F_1$}, ymin=0, ymax=1.1]
                \def\f{figs/csv_for_figs/baseline_plots/baseline_f1_nc.csv}
                \addBaselineEnvelope{\f}
                \addbaselineplot{\colFull}{*}{topic}{\f}{}{0.15} 
                \addbaselineplot{\colCoCo}{square*}{coco}{\f}{}{0.0}
                \addbaselineplot{\colFCIJ}{triangle*}{fci-jci}{\f}{}{0.0}
                \addbaselineplot{\colFCIP}{diamond*}{fci-pooled}{\f}{}{0.0}
                \addbaselineplot{\colLSP}{*}{dai_pooled}{\f}{}{0.0}
                \addbaselineplot{\colLSC}{x}{dai_context}{\f}{}{0.0}
            \end{axis}
        \end{tikzpicture}
    \small (a)
    \end{minipage}\hfill
    \begin{minipage}{0.32\textwidth}
        \centering
        \begin{tikzpicture}
            \begin{axis}[title = {\tiny Subset Recovery}, ylabel={$F_1$}, ymin=0, ymax=1.1]
                \def\f{figs/csv_for_figs/baseline_plots/baseline_subset.csv}
                \addBaselineEnvelope{\f}
                \addbaselineplot{\colFull}{*}{topic}{\f}{}{0.15}
                \addbaselineplot{\colCoCo}{square*}{coco}{\f}{}{0.0}
                \addbaselineplot{\colFCIJ}{triangle*}{fci-jci}{\f}{}{0.0}
                \addbaselineplot{\colFCIP}{diamond*}{fci-pooled}{\f}{}{0.0}
                \addbaselineplot{\colLSP}{*}{dai_pooled}{\f}{}{0.0}
                \addbaselineplot{\colLSC}{x}{dai_context}{\f}{}{0.0}
            \end{axis}
        \end{tikzpicture}
    \small (b)
    \end{minipage}\hfill
    \begin{minipage}{0.32\textwidth}
        \centering
        \begin{tikzpicture}
            \begin{axis}[title = {\tiny Structural Bias Classification}, ylabel={$F_1$}, ymin=0, ymax=1.1,
                legend columns=-1, legend to name=LegBaselines]
                \def\f{figs/csv_for_figs/baseline_plots/baseline_f1.csv}
                \addBaselineEnvelope{\f}
                \addbaselineplot{\colFull}{*}{topic}{\f}{\ourmethod\!($\hat{\mathcal{G}}$)}{0.15}
                \addbaselineplot{\colCoCo}{square*}{coco}{\f}{CoCo($\mathcal{G}^*$)}{0.0}
                \addbaselineplot{\colFCIJ}{triangle*}{fci-jci}{\f}{FCI-JCI}{0.0}
                \addbaselineplot{\colFCIP}{diamond*}{fci-pooled}{\f}{FCI-P}{0.0}
                \addbaselineplot{\colLSP}{*}{dai_pooled}{\f}{LS-P}{0.0}
                \addbaselineplot{\colLSC}{x}{dai_context}{\f}{LS-C}{0.0}
            \end{axis}
        \end{tikzpicture}
    \small (c)
    \end{minipage}

    \vspace{1.5em}
    \centerline{\ref{LegBaselines}}
    \caption{Macro-$F_1$ scores for (a) latent discovery, (b) subset recovery, and (c) structural bias classification. The shaded gray region illustrates the aggregate performance with maximum error bars across all baseline methods for the sake of readability.
    We include all error bars in Appx. Figure \ref{fig:baselines_errors}.} 
    \label{fig:baselines}
\end{figure*}

%% file: figs/tex_figs/sachs.tex
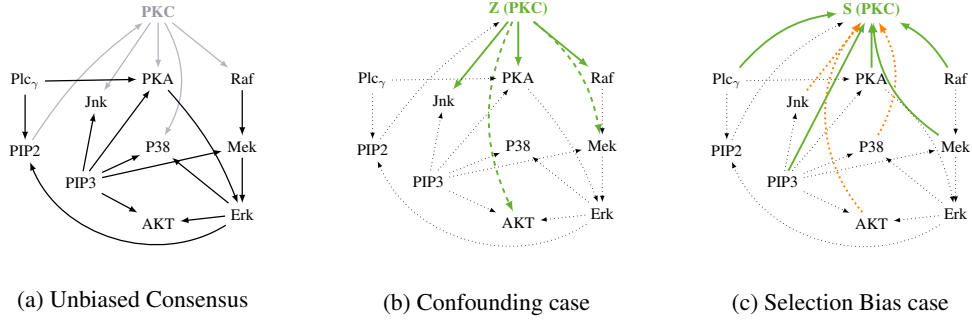
\begin{figure*}[t]
  \centering

  \begin{minipage}{0.33\linewidth}
    \centering
    \scalebox{0.60}{
    \begin{tikzpicture}
      \node [lilacgray] (pkc)  {\textbf{PKC}};
      \node [below left = 1cm and 2cm of pkc] (plcg) {Plc${}_{\gamma}$};
      \node [below = 1cm of plcg] (pip2) {PIP2};
      \node [below right = 0.2cm and 0.3cm of pip2] (pip3) {PIP3};
      \node [below = 1cm of pkc] (pka) {PKA};
      \node [right = 1cm of pka] (raf) {Raf};
      \node [below = 1cm of raf] (mek) {Mek};
      \node [below = 1cm of mek] (erk) {Erk};
      \node [below = 1cm of pka] (p38) {P38};
      \node [below = 2.7cm of pka] (akt) {AKT};
      \node [below left = 1.5cm and 0.5cm of pkc] (jnk) {Jnk};

      \path[-{latex[]}, thick]
      (pkc) edge[lilacgray!70] (pka) (pkc) edge[lilacgray!70] (jnk)
      (pkc) edge[lilacgray!70] (raf) (pkc) edge[lilacgray!70,bend left=30] (p38)
      (pip2) edge[lilacgray!70, bend left=10] (pkc) (plcg) edge (pip2) (plcg) edge (pka)
      (pip3) edge (jnk) (pip3) edge (pka) (pip3) edge (p38) (pip3) edge (mek) (pip3) edge (akt)
      (pka) edge[bend left=10] (erk) (raf) edge (mek) (mek) edge (erk)
      (erk) edge (akt) (erk) edge (p38) (erk) edge[bend left=50] (pip2);
    \end{tikzpicture}
    }\\[1ex] 
    \small (a) Unbiased Consensus 
  \end{minipage}\hfill
  \begin{minipage}{0.33\linewidth}
    \centering
    \scalebox{0.60}{
    \begin{tikzpicture}
      \node [green(ryb)] (pkc)  {\textbf{Z (PKC)}};
      \node [below left = 1cm and 2cm of pkc] (plcg) {Plc${}_{\gamma}$};
      \node [below = 1cm of plcg] (pip2) {PIP2};
      \node [below right = 0.2cm and 0.3cm of pip2] (pip3) {PIP3};
      \node [below = 1cm of pkc] (pka) {PKA};
      \node [right = 1cm of pka] (raf) {Raf};
      \node [below = 1cm of raf] (mek) {Mek};
      \node [below = 1cm of mek] (erk) {Erk};
      \node [below = 1cm of pka] (p38) {P38};
      \node [below = 2.7cm of pka] (akt) {AKT};
      \node [below left = 1.5cm and 0.5cm of pkc] (jnk) {Jnk};

      \path[-{latex[]}, line width=1.2pt, green(ryb)]
      (pkc) edge (pka)
      (pkc) edge (jnk)
      (pkc) edge (raf);
    
      \path[-{latex[]}, line width=1.2pt, green(ryb)]
      (pkc) edge[bend left=20,dashed] (mek)
      (pkc) edge[bend right=25,dashed] (akt);
      
      \path[-{latex[]}, dotted, thin]
      (plcg) edge (pip2) (plcg) edge (pka)
      (pip2) edge[bend left=10] (pkc)
      (pip3) edge (jnk) (pip3) edge (pka) (pip3) edge (p38) (pip3) edge (mek) (pip3) edge (akt)
      (pka) edge[bend left=10] (erk) (raf) edge (mek) (mek) edge (erk)
      (erk) edge (akt) (erk) edge (p38) (erk) edge[bend left=50] (pip2);
    \end{tikzpicture}
    }\\[1ex]
    \small (b) Confounding case 
  \end{minipage}\hfill
  \begin{minipage}{0.33\linewidth}
    \centering
    \scalebox{0.60}{
    \begin{tikzpicture}
      \node [green(ryb)] (pkc) {\textbf{S (PKC)}};
      \node [below left = 1cm and 2cm of pkc] (plcg) {Plc${}_{\gamma}$};
      \node [below = 1cm of plcg] (pip2) {PIP2};
      \node [below right = 0.2cm and 0.3cm of pip2] (pip3) {PIP3};
      \node [below = 1cm of pkc] (pka) {PKA};
      \node [right = 1cm of pka] (raf) {Raf};
      \node [below = 1cm of raf] (mek) {Mek};
      \node [below = 1cm of mek] (erk) {Erk};
      \node [below = 1cm of pka] (p38) {P38};
      \node [below = 2.7cm of pka] (akt) {AKT};
      \node [below left = 1.5cm and 0.5cm of pkc] (jnk) {Jnk};

      \path[-{latex[]}, line width=1.2pt, green(ryb)]
      (pka) edge (pkc)
      (jnk) edge[dotted,orange] (pkc)
      (pip3) edge (pkc)
      (raf) edge[bend right=15] (pkc)
      (p38) edge[dotted,orange,bend right=30] (pkc)
      (plcg) edge[bend left=20] (pkc)
      (akt) edge[bend left=40,dotted,orange] (pkc)
      (mek) edge[bend left=25] (pkc);

      \path[-{latex[]}, dotted, thin]
      (erk) edge (p38) (pip2) edge[bend left=20] (pkc)
      (pip3) edge (pka) (pip3) edge (p38) (pip3) edge (mek)
      (pka) edge[bend left=20] (erk)
      (plcg) edge (pip2) (plcg) edge (pka) (pip3) edge (jnk) (pip3) edge (akt)
      (raf) edge (mek) (mek) edge (erk) (erk) edge (akt)
      (erk) edge[bend left=50] (pip2);
    \end{tikzpicture}
    }\\[1ex]
    \small (c) Selection Bias case
  \end{minipage}

  \caption{ \ourmethod correctly recovers hidden confounding resp. selection bias on the \citet{sachs2005causal} data. (a) Unbiased consensus graph. (b) \ourmethod correctly determines hidden confounding and recovers three directly and two indirectly confounded nodes when PKC is withheld. (c) \ourmethod correctly determines selection bias and recovers six out of eight true ancestors when withholding and selecting on PKC. Green lines are correct (solid) resp. on a causal path (dashed), while orange lines are false positives (dotted). }
  \label{fig:sachs-scenarios}
\end{figure*}

%% file: sections/discussion_conclusion.tex
\section{Discussion \& Conclusion}

We studied the problem of identifying structural biases from observational data. In particular, we showed under which conditions it is possible to infer the presence of latent confounding or selection bias on latent colliders, given data from multiple contexts by detecting dependent mechanism changes. To tell which nodes are subject to what kind of bias in practice, we proposed the \ourmethod algorithm. Through an extensive set of experiments, including a case study on cell-signalling data, we showed that \ourmethod performs well in practice, outperforming the state of the art in detecting confounding, selection bias, and classifying between them, even when our assumptions do not strictly hold. 

\paragraph{Limitations} As with all approaches in causal inference, we too need to make assumptions on how the world works. While we relax the assumptions of sufficiency and i.i.d.-ness, we do require the Markov condition and faithfulness. This means that for systems with cyclic dependencies, our guarantees do not hold. We see this as an important direction for future work. A key assumption in our work is that mechanism changes are sparse and independent. While generally reasonable, this will not always hold in practice, and it will be interesting to generalize our theory given (limited) knowledge about intervention targets, or by including recent results on identifying intervention targets. A practical limitation is that \ourmethod relies on a conditional independence test, and therewith requires sufficiently many i.i.d. samples per context, and similarly, to determine independence of mechanism shifts, requires data from sufficiently many contexts. It will be interesting to extend our theory and method to time-series data, where contexts can possibly be replaced with regimes.

%% file: sections/appendix.tex
\section{Proofs}

Throughout this section, let $G_X$ denote the observed subgraph induced by $\X$, and write $\anc_{G_X}^+(A)=\anc_{G_X}(A)\cup A$ for observed ancestors including the set itself.
For fixed contexts $c,c'$, define the observed local module
\[
    Q_i^c(x_i\mid a;\theta)
    :=P_{\mathrm{obs}}^c\bigl(x_i\mid \pa_{G_X}(X_i)=a;\theta\bigr).
\]
We use the standard analytic zero-set fact that a real analytic function on an open connected parameter set is either identically zero or has a Lebesgue-null zero set.

\subsubsection*{Proof of Theorem~\ref{thm:generic-shift-faithfulness}}
\label{thm:theoretical_identifiability_proof}
\theoreticalidentifiability*

\begin{proof}
Fix the two contexts $c,c'$. We prove two statements: nodes outside the set $A$ have invariant observed local modules for all parameter values, while nodes in $A$ have changed observed local modules except on a null parameter set.

\paragraph{Latent confounding.}
There is no selection, so $P_{\mathrm{obs}}^c=P^c$, and the only changed intrinsic mechanism is the exogenous law of the single latent confounder $Z$. If $X_i\notin \mathrm{ch}_{G^*}(Z)$, then $Z$ is not a parent of $X_i$. Since all other parents of $X_i$ are observed and all observed mechanisms are unchanged, the Markov factorization gives
\[
    Q_i^c(\cdot\mid \pa_{G_X}(X_i);\theta)
    =Q_i^{c'}(\cdot\mid \pa_{G_X}(X_i);\theta)
\]
for every $\theta\in\Theta$. Equivalently, any downstream effect of the shifted $Z$ reaches $X_i$ through observed parents and is blocked by conditioning on them.

If $X_i\in \mathrm{ch}_{G^*}(Z)$, then the latent parent is marginalized out of the observed local module. For parent values $a$ with positive probability,
\[
    Q_i^c(x_i\mid a;\theta)
    =\int p_\theta(x_i\mid a,z)\,dP_\theta^c(z\mid a),
\]
with the same conditional kernel $p_\theta(x_i\mid a,z)$ in both contexts and a context-dependent mixing law for $Z$. By assumption, for this $X_i$ there is some $\theta_0\in\Theta$ at which the two observed local modules differ. Hence, for some $(x_i,a)$ on the common support,
\[
    f_i(\theta):=Q_i^c(x_i\mid a;\theta)-Q_i^{c'}(x_i\mid a;\theta)
\]
is a real analytic function that is not identically zero. Its zero set is therefore Lebesgue null. Thus $X_i$ is in $A_{c,c'}$ for all $\theta$ outside a null set.

\paragraph{Selection.}
Now suppose the single latent variable is a selector $S$. Then there is no latent confounder, all parents of observed variables are observed, and only the selection mechanism changes. If $X_i\notin \anc_{G^*}(S)$, then $S$ is a non-descendant of $X_i$. By the local Markov property,
$X_i\perp S\mid \pa_{G_X}(X_i)$, and therefore also $X_i\perp R(S)\mid \pa_{G_X}(X_i)$. Hence
\[
    Q_i^c(\cdot\mid \pa_{G_X}(X_i);\theta)
    =P^c(\cdot\mid \pa_{G_X}(X_i);\theta)
    =P^{c'}(\cdot\mid \pa_{G_X}(X_i);\theta)
    =Q_i^{c'}(\cdot\mid \pa_{G_X}(X_i);\theta),
\]
where the middle equality uses that no observed intrinsic mechanism changes.

If $X_i\in \anc_{G^*}(S)$, Bayes' rule gives, for any observed parent value $a$ with positive probability,
\[
    Q_i^c(x_i\mid a;\theta)
    =P^c(x_i\mid a,R=1;\theta)
    =\frac{P^c(R=1\mid x_i,a;\theta)P^c(x_i\mid a;\theta)}
           {P^c(R=1\mid a;\theta)}.
\]
The factor $P^c(x_i\mid a;\theta)$ is unchanged across $c,c'$, whereas the selection likelihood changes with the mechanism of $S$. The theorem assumes that this reweighting is nontrivial for at least one $\theta_0$. As above, analyticity then implies that equality of $Q_i^c$ and $Q_i^{c'}$ holds only on a Lebesgue-null subset of $\Theta$.

There are finitely many observed variables. Taking the union of the null exceptional sets over all $X_i\in A$ is still null, while all $X_i\notin A$ are invariant for every $\theta$. Therefore $A_{c,c'}=A$ for all $\theta$ outside a Lebesgue-null subset of $\Theta$.
\end{proof}

\subsubsection*{Proof of Proposition~\ref{lem:expected_coverage}}
\label{lem:expected_coverage_proof}
\expectedcoverage*

\begin{proof}
Let $A=A_{c,c'}$, $k=|A|$, and $d=|\X|$. We write
\[
    \mathrm{cov}_{G_X}(A)=\frac{|A|}{|\anc_{G_X}^+(A)|}.
\]

If $A$ is created by selection, Theorem~\ref{thm:generic-shift-faithfulness} gives $A=\anc_{G^*}(S)\cap\X$, the observed ancestors of the selector. Any observed ancestor of a node in $A$ is again an observed ancestor of $S$, so $\anc_{G_X}^+(A)=A$ and hence $\mathrm{cov}_{G_X}(A)=1$.

If $A$ is created by a latent confounder, Theorem~\ref{thm:generic-shift-faithfulness} gives $A=\mathrm{ch}_{G^*}(Z)$. For $A$ to be ancestrally closed, no observed node outside $A$ may be an observed ancestor of any node in $A$; in particular, there can be no directed edge from $\X\setminus A$ into $A$. In the standard directed Erd\H{o}s--R\'enyi DAG model, each cross pair contributes such an outside-to-inside edge with probability $p/2$. Therefore
\[
    \mathbb{P}\bigl(A\text{ has no outside parent}\mid |A|=k\bigr)
    \leq (1-p/2)^{k(d-k)}
    \leq \exp\{-pk(d-k)/2\}.
\]
On the complementary event, $\anc_{G_X}^+(A)$ strictly contains $A$, so $\mathrm{cov}_{G_X}(A)<1$. Thus
\[
    \mathbb{P}\bigl(\mathrm{cov}_{G_X}(A)<1\mid |A|=k\bigr)
    \geq 1-\exp\{-pk(d-k)/2\}
    =1-O(e^{-pk(d-k)/2}),
\]
which proves the claimed bound.
\end{proof}

\subsubsection*{Proof of Proposition~\ref{cor:lv_minimality}}
\lvminimality*

\begin{proof}
Let $A=A_{c,c'}$. The sink nodes of $A$ are the maximal nodes of the subgraph induced by $A$,
\[
    M:=\{X_i\in A: \mathrm{ch}_{G_X}(X_i)\cap A=\emptyset\}.
\]
Because $A$ is ancestrally closed, every node in $A$ is an ancestor of some node in $M$, and every observed ancestor of a node in $M$ lies in $A$. Hence
\[
    A=\anc_{G_X}^+(M).
\]
A single selector $S$ with $\pa(S)=M$ therefore generates exactly the observed shifted set $A$ by Theorem~\ref{thm:generic-shift-faithfulness}, using $|M|$ edges into the latent node.

No single-selector explanation can use fewer such edges. Indeed, suppose another selector parent set $B$ also satisfies $\anc_{G_X}^+(B)=A$. For any $m\in M$, since $m\in A=\anc_{G_X}^+(B)$, either $m\in B$ or $m$ is an ancestor of some $b\in B$. In the latter case $b\in A$ is a descendant of $m$; maximality of $m$ in $A$ forces $b=m$. Thus $m\in B$, so $M\subseteq B$.

A single-confounder explanation is local by Theorem~\ref{thm:generic-shift-faithfulness}: to generate exactly $A$, the latent confounder must be a parent of every node in $A$, and hence uses $|A|$ latent edges. Since $A$ contains at least one parent-child pair, not every node of $A$ is maximal, so $|M|<|A|$. Therefore the selector with parent set $M$ is edge-minimal among all single-latent-node explanations.
\end{proof}

\paragraph{Coverage bound used for thresholding.}
If a confounded set $A$ has size $k$ and is not ancestrally closed, then $\anc_{G_X}^+(A)$ contains at least one node outside $A$. Consequently,
\[
    \mathrm{cov}_{G_X}(A)
    =\frac{|A|}{|\anc_{G_X}^+(A)|}
    \leq \frac{k}{k+1}.
\]
This gives a simple finite-size thresholding heuristic for separating low-coverage confounding patterns from ancestrally closed selection patterns.

\section{Calculating Adjusted Mutual Information (AMI)}\label{sec:appx_alg}
We show here, how we compute the AMI score for two partitions, $\Pi^\text{obs}_i$ and $\Pi^\text{obs}_j$.
Following \cite{mameche2024identifying}, each variable $V_i$ induces a partition $\Pi^\text{obs}_i$ of the context set $\Cc$, where two contexts lie in the same block if the conditional mechanism of $V_i$ is invariant across them. To quantify whether two variables exhibit coordinated mechanism shifts, we compare their induced partitions.

  Let $\Pi^\text{obs}_i = \{\pi_i^1, \ldots, \pi_i^I\}$ and $\Pi^\text{obs}_j = \{\pi_j^1, \ldots, \pi_j^J\}$ be two partitions of $\Cc$.
  Their overlap is summarized by the contingency table $\mathbf{N} = (n_{ab})$, where
  \[
    n_{ab} = \left|\pi_i^a \cap \pi_j^b\right|, \qquad
    u_a = \sum_b n_{ab}, \qquad
    v_b = \sum_a n_{ab}, \qquad
    N = |\Cc|.
  \]
  As in \cite{mameche2024identifying}, the mutual information between the two partitions is
  \[
    \operatorname{MI}(\Pi^\text{obs}_i, \Pi^\text{obs}_j)
    = \sum_{a=1}^{I} \sum_{b=1}^{J}
    \frac{n_{ab}}{N} \log \frac{n_{ab}N}{u_a v_b},
  \]
  with entropies
  \[
    H(\Pi^\text{obs}_i) = - \sum_{a=1}^{I} \frac{u_a}{N} \log \frac{u_a}{N},
    \qquad
    H(\Pi^\text{obs}_j) = - \sum_{b=1}^{J} \frac{v_b}{N} \log \frac{v_b}{N}.
  \]
  Large values indicate that knowing the shift pattern of one variable is informative about the shift pattern of the other, which is precisely the signal we expect when both are jointly affected by the same confounder or collider.

  With finitely many contexts, however, raw mutual information is positively biased even for unrelated partitions. We therefore use adjusted mutual information (AMI) \citep{vinh2010information}, which subtracts the expected mutual information under random partitions and normalizes the result:
  \[
    \operatorname{AMI}(\Pi^\text{obs}_i, \Pi^\text{obs}_j)
    = \frac{\operatorname{MI}(\Pi^\text{obs}_i, \Pi^\text{obs}_j) - \mathbb{E}[\operatorname{MI}(\Pi^\text{obs}_i, \Pi^\text{obs}_j)]}
    {\tfrac{1}{2}\left(H(\Pi^\text{obs}_i) + H(\Pi^\text{obs}_j)\right) - \mathbb{E}[\operatorname{MI}(\Pi^\text{obs}_i, \Pi^\text{obs}_j)]}.
  \]
  This yields $\operatorname{AMI} = 1$ for identical partitions and an expected value of $0$ for independent random partitions, making it a more reliable finite-sample affinity score than raw mutual information. The expectation $\mathbb{E}[\operatorname{MI}(\Pi^\text{obs}_i, \Pi^\text{obs}_j)]$ is computed according to the hypergeometric distribution.

  As a simple example, let $\Cc = \{c_1, c_2, c_3, c_4\}$ and consider
  \[
    \Pi^\text{obs}_i = \{\{c_1, c_2\}, \{c_3, c_4\}\},
    \qquad
    \Pi^\text{obs}_j = \{\{c_1, c_2\}, \{c_3, c_4\}\}.
  \]
  Then the contingency table is
  \[
    \mathbf{N} =
    \begin{pmatrix}
        2 & 0 \\
    0 & 2
  \end{pmatrix},
    \]
    so the two partitions agree exactly and $\operatorname{AMI}(\Pi^\text{obs}_i, \Pi^\text{obs}_j) = 1$. In contrast, if $\Pi^\text{obs}_j = \{\{c_1, c_3\}, \{c_2, c_4\}\}$, the overlap is spread evenly across the contingency table, the mutual information vanishes, and the adjusted score is correspondingly near $0$.

  In our algorithm, the pairwise AMI values define the affinity matrix $\ContingAB$, whose connected components are candidate subsets of variables that shift together.

\section{Additional Results on Synthetic Data}
\label{sec:extra_details}
\begin{table}[h]
    \centering
    \caption{Experimental Parameters for Synthetic Data Generation}
    \begin{tabular}{lp{10cm}}
        \toprule
        \textbf{Parameter} & \textbf{Description} \\
        \midrule
        $n$  & The number of samples (observations) available within each context. \\[1.5ex]
        
        $m$ & The total number of observed variables in the causal Directed Acyclic Graph (DAG). \\[1.5ex]
        
        
        $r$ & The ratio of mechanism shifts assigned to observed nodes relative to latent nodes. \\[1.5ex]
        
        $\textit{mcf}$ & Scaling factor for context availability, where the number of contexts $m_c = m \times \textit{mcf}$. \\[1.5ex]
        
        $\delta$  & The shift density; the ratio of contexts containing a shift to the total number of contexts (baseline $\delta = 0.48$). \\
        \bottomrule
    \end{tabular}
    \label{tab:parameters}
\end{table}
\paragraph{Synthetic Data Generation Process}
We generate $25$ instances of ground truth DAGs across for each node size, $n \in \{3, 5, 7, 9, 11, 13\}$. 
For our main experiments, we set the total number of contexts $m_c$ to be $2n$, with $500$ samples per context.
Within these contexts, we allow $0.48 * m_c$ contexts (to ensure sparsity) to appear to be shifted in order to ensure sparse mechanism shifts.
We set the ratio of the number of shifts per each observed node to the number of shifts per each latent node to be $0.25$.

To ensure exclusive mechanism shifts while maintaining a controlled graph density, we define the shift budget allocation for latent and observed nodes as follows. 
The total shift budget $B_{total}$ is defined as:
\[ B_{total} = \lfloor m_c \cdot \delta \rfloor \]
Given a set of latent nodes $S_{l}$ (comprising confounders and colliders) and a set of observed nodes $X$, we define the effective denominator $D$ using the observed shift ratio $r$:
\[ D = |S_{l}| + (r \cdot |X|) \]
The number of shifts assigned to each latent node ($V$) and each observed node ($X_i$) is then calculated as:
\[ V = \max\left(1, \left\lfloor \frac{B_{total}}{D} \right\rfloor\right) \]
\[ X_i = \lfloor r \cdot V \rfloor \]


To rigorously evaluate the proposed method, we conduct a series of controlled ablations across several key dimensions. These experiments are designed to test the robustness of our information-theoretic signatures under varying structural and environmental constraints.

\subsection{Parameter Definitions}
The synthetic data generation and subsequent evaluations are governed by the parameters detailed in Table~\ref{tab:parameters}. For all ablations, we vary the hyperparameter in question over the defined range, and fix the other 
parameters in their default values, i.e., $n = 500$, $m = 5$, our main oracle is \ourmethod\!($\mathcal{G}^*$), $r = 0.25$, $mcf = 3.0$, and $\delta = 0.48$.

\subsection{Ablation on Sample Complexity ($n$)}\label{sec:samp_comp}
We evaluate the data efficiency and asymptotic stability of the AMI and coverage signatures by varying $n \in \{100, 200, 500, 1000\}$. This ablation is performed across a range of DAG sizes ($m \in \{3, \dots, 15\}$) to determine 
the optimal number of samples required per context for reliable results. As we see in Fig. \ref{fig:ablation_samples}, we get reliable results with lower number of samples. However, we stick to $500$ samples as the variance is low and 
runtime is not high. 

The final classification is governed by a hierarchical search for optimal thresholds based on AMI and Coverage. A threshold of 0.1 on the AMI values suffices for distinguishing the presence of a latent variable.
It is important to note that the thresholds (lower) obtained for $n =100, 200$  are lower, $0.46, 0.48$ respectively, for smaller sample sizes and higher for $n =500, 1000$ ($0.60, 0.71$ respectively) which is expected as the estimates of mutual information are less reliable with fewer samples. 
Since we have $n=500$, we have $t^* = 0.6$. Subsets with coverage $< t^*$ are classified as \figbf{Confounded} as they are more prone to having localized mechanism influence, while those exceeding $t^*$ suggest \figbf{Selection Bias} due to widespread influence across upstream parents.

\input{figs/tex_figs/samples_new.tex}

\subsection{Ablation on the Structural Shift Ratio ($r$)}
This experiment investigates the impact of structural bias by varying the frequency of shifts in observed versus latent variables using $r \in \{0, 0.25, 0.5, 1.0, 1.5\}$. 
We assess the method's capability to isolate latent selection effects when observed nodes exhibit varying levels of mechanism volatility for graphs of size 5.
As we can see in in Fig. \ref{fig:shifts_ablation}, as $r$ increases, the performance of the method degrades as the observed nodes become more volatile and the 
signatures become less exclusive to latent shifts. \footnote{An interesting note here is that, since we choose a low value for $r$, it subdues the noise arising from backdoors in the underlying causal graph. If the ratio is higher, then we may misclassify 
confounding as selection bias in many cases. This is because, due to the presence of latent variables, while looking at conditional distributions, an observed parent that is conditioned on might open up a latent v-path, leading to unforeseen correlations. This would lead to noise as well.
Hence, having a low $r$ ratio also helps us in reducing the chance of seeing spurious effects via backdoors.}

\input{figs/tex_figs/shifts_new.tex}

\subsection{Ablation on the Shift Density ($\delta$)}
We test the robustness of \ourmethod against the sparsity assumption. We ablate over shift densities $\delta \in \{0.2, 0.48, 0.6, 0.8\}$, which denote the fraction of contexts in which any causal mechanism shifts occur, across the entire set of variables, $V$. 
Higher values of $\delta$ imply noisy mechanism shifts. We show in Fig. \ref{fig:sparsity_ablation} that \ourmethod works well even after relaxing sparsity assumptions.

\input{figs/tex_figs/sparse_new.tex}

\subsection{Ablation on number of Contexts ($m_c$)}
We perform \ourmethod's effectiveness across the number of given contexts. We vary $\textit{mcf} \in \{0.5, 1.0, 2.0, 3.0\}$ which is the ratio between the number of observed nodes in the graph, and the number of contexts $m_c$. 
More the number of contexts, less the number of noisy mechanism shifts detected. As we see in Fig. \ref{fig:contexts_ablation}, higher the ratio, better the performance.

\input{figs/tex_figs/contexts_new.tex}

\input{figs/tex_figs/perturbed.tex}



\subsection{\ourmethod's performance on incorrect graphs}\label{sec:perturbed}
To simulate structural errors in causal discovery (e.g., from a search algorithm), we introduce noise into the DAG via edge-flipping. We evaluate the \textit{Perturbed} oracle at $m=5$ with flip fractions, $ f\in \{0.1, 0.25, 0.4, 0.5\}$ in Fig. \ref{fig:perturbed_graph_ablation}. 
This measures the degree of structural degradation the algorithm can tolerate before the identification of latent types becomes unreliable.


\subsection{\ourmethod and its variants} \label{sec:oracles_appendix}
To rigorously evaluate the framework components, we define three oracles:
\begin{itemize}
    \item \figbf{\ourmethod\!($\Pi^\text{obs}$):} This variant receives the partitions and then performs classification.
    \item \figbf{\ourmethod\!($\mathcal{G}^*$):} This variant operates on the true underlying causal graph.
    \item \figbf{\ourmethod\!($\hat{\mathcal{G}}$):} This variant operates on a learned causal graph using \topic.
\end{itemize}
We showcase the performance on structural bias classification of these variants according to different metrics in Fig. \ref{fig:oracles_extra}. The runtimes for each oracle across different DAG sizes are given in Table \ref{tab:runtime_results}. \\

\input{figs/tex_figs/main_ablation_extras.tex}
\begin{table}[h]
\centering
\caption{Runtimes across variants and graph sizes ($n$).}
\label{tab:runtime_results}
\small 
\begin{tabular}{lccclcc} 
\toprule
\textbf{Oracle} & \textbf{$n$} & \textbf{Runtime (s)} & & \textbf{Oracle} & \textbf{$n$} & \textbf{Runtime (s)} \\
\cmidrule{1-3} \cmidrule{5-7}

\multirow{7}{*}{full} 
& 3  & $0.19 \pm 0.01$ & & \multirow{7}{*}{perturbed} 
& 3  & $0.38 \pm 0.02$ \\
& 5  & $1.40 \pm 0.31$ & & & 5  & $2.89 \pm 0.55$ \\
& 7  & $4.51 \pm 0.72$ & & & 7  & $9.24 \pm 1.02$ \\
& 9  & $10.79 \pm 0.82$ & & & 9  & $21.46 \pm 1.49$ \\
& 11 & $19.92 \pm 1.37$ & & & 11 & $40.70 \pm 2.82$ \\
& 13 & $34.47 \pm 2.56$ & & & 13 & $68.37 \pm 4.48$ \\
\midrule
\multirow{7}{*}{\topic} 
& 3  & $0.27 \pm 0.40$ & & \multirow{7}{*}{partitions} 
& 3  & $0.38 \pm 0.78$ \\
& 5  & $1.39 \pm 0.34$ & & & 5  & $0.13 \pm 0.13$ \\
& 7  & $4.35 \pm 0.70$ & & & 7  & $0.21 \pm 0.15$ \\
& 9  & $10.96 \pm 0.73$ & & & 9  & $0.31 \pm 0.18$ \\
& 11 & $19.39 \pm 1.47$ & & & 11 & $0.37 \pm 0.24$ \\
& 13 & $35.13 \pm 2.77$ & & & 13 & $0.51 \pm 0.24$ \\
\bottomrule
\end{tabular}
\end{table}

\begin{algorithm}[t!]
    \caption{$\ourmethod$ Suite: Discovery and Classification}
    \label{alg:conjoined_suite}
    
    \textbf{Sub-procedure A:} $\ourmethod (\Pi^\text{obs})$ \\
    \Input{ Partitions $\Pi^\text{obs}$ for every observed variable}
    \Output{ Subsets of $\X$ suffering from selection bias or confounding.}  
    \BlankLine
    Construct an affinity matrix $\ContingAB$ where $\ContingAB_{ij} = \text{AMI}(\Pi^\text{obs}_i, \Pi^\text{obs}_j)$\;
    Identify subsets $X_S \subseteq \X$ as connected components in $\ContingAB$\;
    \eIf{$X_S = \emptyset$}{ \Return{Unbiased} }{
        \ForEach{identified subset $X_S$}{
            Calculate coverage $\operatorname{cov}(X_S) \geq \tau$ ? \figbf{Selection Bias} : \figbf{Confounding}\;
        }
        \Return{Subsets $X_S$ with labels}\;
    }
    
    \vspace{0.5em} \hrule \vspace{0.5em} 
    
    \textbf{Sub-procedure B:} $\ourmethod\!(\mathcal{G})$ \\
    \Input{ Observed data over $\X, C$; observed causal graph $\Gc$.}
    \ForEach{variable $X_i \in \X$}{
        Compute $p_{c,c'}$ and convert to partition $\Pi^\text{obs}_i$\;
    }
    \Return{\ourmethod\!($\Pi$)}\;

    \vspace{0.5em} \hrule \vspace{0.5em}

    \textbf{Sub-procedure C:} $\ourmethod\!(\hat{\mathcal{G}})$ \\
    \Input{Observed data over $\X, C$}
    $\hat{\Gc}$ = \topic({$\X, C$})\;
    \Return{\ourmethod\!($\hat{\Gc}$)}\;
\end{algorithm}

\subsection{Analysis of runtimes}


All experiments were conducted on an Apple M3 Pro (14-core CPU, 36GB RAM), utilizing 10 CPU cores to parallelize all context evaluations via Python’s multiprocessing. 
Our main oracle, \mbox{\ourmethod\!($\mathcal{G^*}$)} exhibits a mean runtime of approximately 34s for 13-node graphs. The \ourmethod\!($\Pi$) oracle takes less than 1s across all sizes. The \ourmethod\!($\hat{\mathcal{G}}$) 
variant, where the underlying causal graph is estimated using \topic, closely tracks the  \ourmethod\!($\mathcal{G^*}$)'s performance. As for the baselines, the latent selection variants (LS-Pooled, LS-Contexts) and 
FCI variants are very fast, with sub-second execution times even for the larger graphs, CoCo scales much more aggressively, and is slower than \ourmethod\!($\mathcal{G^*}$). 
The runtimes for each oracle are shown in Table \ref{tab:runtime_results} and the baselines in Fig. \ref{fig:runtime_comparison} in Section \ref{sec:extra_details} of the Appendix.

\subsection{Additional Details on Baselines}\label{sec:ls_deets}
\paragraph{LS-Pooled}
The first baseline \citep{dai2025selection}, treats the non-i.i.d. dataset as a unified, exchangeable sample by marginalizing over all context labels. In this approach, causal discovery is performed in a single pass over the aggregated data matrix, which maximizes the sample size to increase the statistical power of the underlying conditional independence and rank tests. 
\paragraph{LS-Contexts}
The second baseline \citep{dai2025selection}, explicitly performs structural discovery for each context. After this, the latent structure recovered in a majority of the contexts is retained.
\paragraph{Additional Comparisons with Baselines}\label{sec:baselines_appendix}
We show some extra results on the baselines in this section. Fig. \ref{fig:baselines_errors} contains the reported error bars of the baselines in comparison with \ourmethod. Fig. \ref{fig:baseline_metrics} gives the accuracy, precision, and recall on the task of structural bias classification,
Fig. \ref{fig:latent_selection_comparison} compares \ourmethod with the \citep{dai2025selection} variants on their data assumptions, and Fig. \ref{fig:runtime_comparison} compares the runtime of all baselines.

\input{figs/tex_figs/baselines_new_fill.tex}
\vspace{1.5em}

\input{figs/tex_figs/baselines_new_extras.tex}
\vspace{1.5em}

\input{figs/tex_figs/latent_selection_vs_topic.tex}
\vspace{1.5em}

\input{figs/tex_figs/runtime_baselines.tex}


\section{Additional Results on the Cell-Signaling Data} 
\label{sec:extra_sachs}
The \citet{sachs2005causal} dataset is particularly interesting as it contains multiple contexts with known interventions. However, it is also very challenging as the system is complex and has many cyclic dependencies~\cite{meinshausen2016methods} and hence these interventions might affect multiple nodes in overlapping contexts, leading to noisy partitions. 
In the unbiased setting, \ourmethod finds overlaps in causal mechanisms of Akt-Mek, Plcg-P38, and PKA-PKC. This can be due to biological crosstalk.

For instance, the dependencies detected between \mbox{Akt-Mek} may be attributed to the extensive pathway crosstalk between the Ras-MAPK and PI3K-mTORC1 cascades \citep{mendoza2011ras}. 
Similarly, the identified dependencies between \mbox{PKA-PKC} likely reflect shared regulation via phospholipase C(PLC) signaling; such simultaneous activation pathways can render conditioning on \mbox{PLC$\gamma$} 
alone insufficient for isolation \citep{nishizuka1995protein}. Finally, the overlapping shifts in \mbox{Plcg-P38} may reflect parallel activation by common upstream MAP3Ks in response to shared extracellular 
growth or stress signals \citep{canovas2021diversity}.



The mechanism shifts we obtain are given in Table \ref{tab:mechanism_shifts} and the retrieved latent structures by the other baselines in Fig. \ref{fig:sachs-baselines}.

\vspace{1em}

\begin{table}[h]
    \centering
    \small
    \renewcommand{\arraystretch}{1.2}
    \begin{tabular}{ll}
        \toprule
        \textbf{Node} & \textbf{Causal Mechanism Shifts} \\
        \midrule
        Akt  & \textbf{C0} (T cell activation), \textbf{C2} (AKT inhibitor) \\
        Erk  & \textbf{C3} (G06976), \textbf{C8} (beta2cAMP) \\
        Jnk  & No shifts \\
        Mek  & \textbf{C0} (activation), \textbf{C2} (AKT inhibitor), \textbf{C4} (MEK1/2), \textbf{C6} (LY294002) \\
        P38  & \textbf{C1} (ICAM-2/LFA-1), \textbf{C7} (PMA/PIP2 inhibition) \\
        PIP2 & \textbf{C3} (G06976), \textbf{C5} (U0126) \\
        PIP3 & \textbf{C0, C1, C2, C3, C4}, \textbf{C7} \\
        PKA  & \textbf{C3} (AKT inhibition), \textbf{C4} (MEK inhibition) \\
        PKC  & \textbf{C4} (Psitectorigenin) \\
        Plcg & \textbf{C1} (LFA-1), \textbf{C7} (PIP2 inhibition), \textbf{C8} (beta2cAMP) \\
        Raf  & \textbf{C0} (activation), \textbf{C4} (MEK inhibition), \textbf{C7} (PMA), \textbf{C8} (beta2cAMP) \\
        \bottomrule
    \end{tabular}
    \vspace{1em}
    \caption{Detailed inventory of causal mechanism shifts identified across the signaling network nodes.}
    \label{tab:mechanism_shifts}
\end{table}

\vspace{3em}

\input{figs/tex_figs/sachs_coco.tex}


%% file: figs/tex_figs/samples_new.tex
\begin{figure*}[t]
    \centering
    \renewcommand{\prtickfontsize}{\tiny}
    \renewcommand{\prlabelfontsize}{\tiny}
    \renewcommand{\prlegendfontsize}{\tiny}

    \pgfplotsset{
        every axis/.append style={
            pretty line, 
            width=1.1\linewidth, 
            height=2.8cm,
            unbounded coords=discard,
            restrict x to domain=0:13, 
            xmax=13,                   
            clip=true,                 
            xlabel={Dag Size $\lvert\mathcal{G}\rvert$},
            mark size=0.7pt,
            y label style={font=\prlabelfontsize},
            x label style={font=\prlabelfontsize},
        }
    }

    \newcommand{\addstdplot}[5]{
        \addplot+[name path=upper, draw=none, forget plot] table[col sep=comma, x=X, y expr=\thisrow{#3}+\thisrow{#3_std}] {#4};
        \addplot+[name path=lower, draw=none, forget plot] table[col sep=comma, x=X, y expr=\thisrow{#3}-\thisrow{#3_std}] {#4};
        \addplot+[fill opacity=0.15, forget plot] fill between[of=upper and lower];
        \addplot+[thick, mark=#2] table[col sep=comma, x=X, y=#3] {#4}; 
        \ifx\&#5\&\else\addlegendentry{#5}\fi
    }

    \begin{minipage}{0.31\textwidth}
        \centering
        \begin{tikzpicture}
            \begin{axis}[
                ylabel={Class. $F_1$}, ymin=0, ymax=1.1,
                legend columns=4,
                legend to name=CommonLegendRow1
            ]
                \addstdplot{0}{*}{100}{figs/csv_for_figs/samples_ablations/samples_f1.csv}{N=100}
                \addstdplot{1}{square*}{200}{figs/csv_for_figs/samples_ablations/samples_f1.csv}{N=200}
                \addstdplot{2}{triangle*}{500}{figs/csv_for_figs/samples_ablations/samples_f1.csv}{N=500}
                \addstdplot{3}{diamond*}{1000}{figs/csv_for_figs/samples_ablations/samples_f1.csv}{N=1000}
            \end{axis}
        \end{tikzpicture}
    \end{minipage}
    \begin{minipage}{0.31\textwidth}
        \centering
        \begin{tikzpicture}
            \begin{axis}[ylabel={Subset $F_1$}, ymin=0, ymax=1.1]
                \addstdplot{0}{*}{100}{figs/csv_for_figs/samples_ablations/samples_subset.csv}{}
                \addstdplot{1}{square*}{200}{figs/csv_for_figs/samples_ablations/samples_subset.csv}{}
                \addstdplot{2}{triangle*}{500}{figs/csv_for_figs/samples_ablations/samples_subset.csv}{}
                \addstdplot{3}{diamond*}{1000}{figs/csv_for_figs/samples_ablations/samples_subset.csv}{}
            \end{axis}
        \end{tikzpicture}
    \end{minipage}

    \vspace{0.6em}
    \centerline{\ref{CommonLegendRow1}} 

    \begin{minipage}{0.24\textwidth}
        \centering
        \begin{tikzpicture}
            \begin{axis}[
                xlabel={Threshold $t \, (N=100)$}, ylabel={Rate}, ymin=0, ymax=1.1, xmin=0, xmax=1,
                restrict x to domain=0:1,
                legend to name=CommonLegendRow4,
                legend columns=3
            ]
                \def\file{figs/csv_for_figs/samples_ablations/samples_threshold_sweep_100.csv}
                \addplot+[thick, mark=none] table[x=threshold, y=tpr_confounder, col sep=comma] {\file}; 
                \addlegendentry{TPR Conf.}
                \addplot+[thick, mark=none] table[x=threshold, y=tpr_collider, col sep=comma] {\file}; 
                \addlegendentry{TPR Sel.}
                \addplot+[dashed, ultra thick, mark=none, pr-color-gray4] table[x=threshold, y=f1_macro, col sep=comma] {\file}; 
                \addlegendentry{Macro $F_1$}
                
                \draw[pr-color-gray4, thick, dotted] (axis cs:0.46, 0) -- (axis cs:0.46, 1.1) 
                    node[pos=0.9, right, font=\tiny] {$t^*$};
            \end{axis}
        \end{tikzpicture}
    \end{minipage}\hfill
    \begin{minipage}{0.24\textwidth}
        \centering
        \begin{tikzpicture}
            \begin{axis}[
                xlabel={Threshold $t \, (N=200)$}, ylabel={Rate}, ymin=0, ymax=1.1, xmin=0, xmax=1,
                restrict x to domain=0:1
            ]
                \def\file{figs/csv_for_figs/samples_ablations/samples_threshold_sweep_200.csv}
                \addplot+[thick, mark=none] table[x=threshold, y=tpr_confounder, col sep=comma] {\file}; 
                \addplot+[thick, mark=none] table[x=threshold, y=tpr_collider, col sep=comma] {\file}; 
                \addplot+[dashed, ultra thick, mark=none, pr-color-gray4] table[x=threshold, y=f1_macro, col sep=comma] {\file}; 
                
                \draw[pr-color-gray4, thick, dotted] (axis cs:0.48, 0) -- (axis cs:0.48, 1.1) 
                    node[pos=0.9, right, font=\tiny] {$t^*$};
            \end{axis}
        \end{tikzpicture}
    \end{minipage}\hfill
    \begin{minipage}{0.24\textwidth}
        \centering
        \begin{tikzpicture}
            \begin{axis}[
                xlabel={Threshold $t \, (N=500)$}, ylabel={Rate}, ymin=0, ymax=1.1, xmin=0, xmax=1,
                restrict x to domain=0:1
            ]
                \def\file{figs/csv_for_figs/samples_ablations/samples_threshold_sweep_500.csv}
                \addplot+[thick, mark=none] table[x=threshold, y=tpr_confounder, col sep=comma] {\file}; 
                \addplot+[thick, mark=none] table[x=threshold, y=tpr_collider, col sep=comma] {\file}; 
                \addplot+[dashed, ultra thick, mark=none, pr-color-gray4] table[x=threshold, y=f1_macro, col sep=comma] {\file}; 
                
                \draw[pr-color-gray4, thick, dotted] (axis cs:0.60, 0) -- (axis cs:0.60, 1.1) 
                    node[pos=0.9, right, font=\tiny] {$t^*$};
            \end{axis}
        \end{tikzpicture}
    \end{minipage}\hfill
    \begin{minipage}{0.24\textwidth}
        \centering
        \begin{tikzpicture}
            \begin{axis}[
                xlabel={Threshold $t \, (N=1000)$}, ylabel={Rate}, ymin=0, ymax=1.1, xmin=0, xmax=1,
                restrict x to domain=0:1
            ]
                \def\file{figs/csv_for_figs/samples_ablations/samples_threshold_sweep_1000.csv}
                \addplot+[thick, mark=none] table[x=threshold, y=tpr_confounder, col sep=comma] {\file}; 
                \addplot+[thick, mark=none] table[x=threshold, y=tpr_collider, col sep=comma] {\file}; 
                \addplot+[dashed, ultra thick, mark=none, pr-color-gray4] table[x=threshold, y=f1_macro, col sep=comma] {\file}; 
                
                \draw[pr-color-gray4, thick, dotted] (axis cs:0.71, 0) -- (axis cs:0.71, 1.1) 
                    node[pos=0.9, right, font=\tiny] {$t^*$};
            \end{axis}
        \end{tikzpicture}
    \end{minipage}

    \vspace{0.8em}
    \centerline{\ref{CommonLegendRow4}}

    \caption{Ablation results for Sample Complexity across DAG sizes. The average execution runtimes are 1.57s, 2.91s, 16.79s, and 124.23s for $N=100, 200, 500,$ and $1000$, respectively. 
    We choose $n = 500$ for our main experiments as it provides a good balance of performance and runtime and matches base requirements for the baselines. The structural bias classification and subset recovery $F1$ scores (upper) are similar.}
    \label{fig:ablation_samples}
\end{figure*}
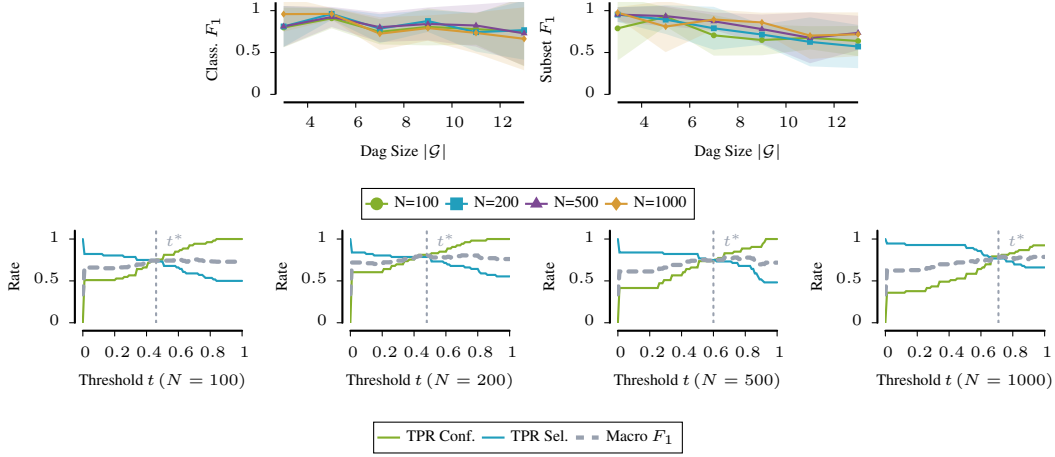

%% file: figs/tex_figs/shifts_new.tex
\begin{figure*}[t]
    \centering
    \renewcommand{\prtickfontsize}{\tiny}
    \renewcommand{\prlabelfontsize}{\tiny}
    \renewcommand{\prlegendfontsize}{\fontsize{4}{5}\selectfont} 

    \pgfplotsset{
        compat=1.18,
        every axis/.append style={
            pretty line,
            scale only axis,       
            width=0.65\linewidth,   
            height=1.4cm,          
            unbounded coords=discard,
            mark size=0.6pt,       
            clip mode=individual,  
            clip=true,             
            axis on top,           
            xlabel style={font=\prlabelfontsize, at={(ticklabel cs:0.5)}, anchor=north, yshift=2pt},
            ylabel style={font=\prlabelfontsize, at={(ticklabel cs:0.5)}, anchor=south, yshift=-2pt},
            tick label style={font=\prtickfontsize},
            legend style={
                font=\prlegendfontsize, 
                at={(0.98,0.98)}, 
                anchor=north east, 
                cells={anchor=west},
                fill=white, 
                fill opacity=0.8, 
                draw opacity=1,
                row sep=-2pt, 
                inner sep=1pt
            }
        },
        bottom x/.style={xlabel={Ratio $r$}}
    }

    \newcommand{\addstdplot}[5]{
        \addplot+[name path=upper, draw=none, forget plot, opacity=0] table[col sep=comma, x=X, y expr=\thisrow{#3}+\thisrow{#3_std}] {#4};
        \addplot+[name path=lower, draw=none, forget plot, opacity=0] table[col sep=comma, x=X, y expr=\thisrow{#3}-\thisrow{#3_std}] {#4};
        \addplot+[fill opacity=0.15, forget plot, draw=none] fill between[of=upper and lower];
        \addplot+[thick, mark=#2] table[col sep=comma, x=X, y=#3] {#4}; 
        \ifx\&#5\&\else\addlegendentry{#5}\fi
    }

    \begin{minipage}{0.24\textwidth}
        \centering
        \begin{tikzpicture}
            \begin{axis}[ylabel={Class. $F_1$}, ymin=0, ymax=1.1, bottom x]
                \def\f{figs/csv_for_figs/shifts_ablations/shifts_f1.csv}
                \addstdplot{0}{*}{full}{\f}{\ourmethod\!($\mathcal{G}^*$)} 
            \end{axis}
        \end{tikzpicture}
    \end{minipage}\hfill
    \begin{minipage}{0.24\textwidth}
        \centering
        \begin{tikzpicture}
            \begin{axis}[ylabel={Subset $F_1$}, ymin=0, ymax=1.1, bottom x]
                \def\f{figs/csv_for_figs/shifts_ablations/shifts_subset.csv}
                \addstdplot{0}{*}{full}{\f}{\ourmethod\!($\mathcal{G}^*$)}
            \end{axis}
        \end{tikzpicture}
    \end{minipage}\hfill
    \begin{minipage}{0.24\textwidth}
        \centering
        \begin{tikzpicture}
            \begin{axis}[ylabel={AMI}, ymin=0, ymax=1.1, bottom x]
                \def\file{figs/csv_for_figs/shifts_ablations/shifts_ami_full.csv}
                \addstdplot{1}{*}{none}{\file}{Unbiased}
                \addstdplot{2}{square*}{confounder}{\file}{Conf.}
                \addstdplot{3}{triangle*}{collider}{\file}{Sel.}
            \end{axis}
        \end{tikzpicture}
    \end{minipage}\hfill
    \begin{minipage}{0.24\textwidth}
        \centering
        \begin{tikzpicture}
            \begin{axis}[ylabel={Coverage}, ymin=0, ymax=1.1, bottom x]
                \def\file{figs/csv_for_figs/shifts_ablations/shifts_coverage_full.csv}
                \addstdplot{1}{*}{none}{\file}{Unbiased}
                \addstdplot{2}{square*}{confounder}{\file}{Conf.}
                \addstdplot{3}{triangle*}{collider}{\file}{Sel.}
            \end{axis}
        \end{tikzpicture}
    \end{minipage}

    \caption{Ablation over values of ratio $r$. As $r$ increases, performance metrics decline, indicating higher sensitivity to observed node noise.}
    \label{fig:shifts_ablation}
\end{figure*}
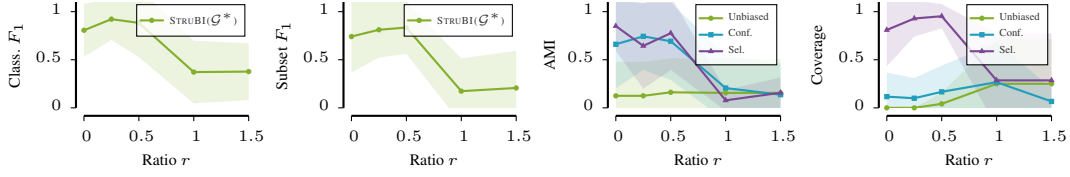

%% file: figs/tex_figs/sparse_new.tex
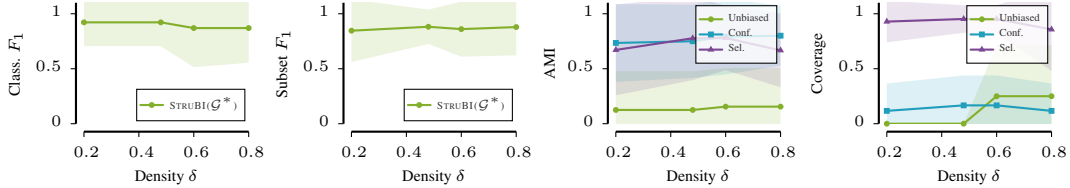
\begin{figure*}[t]
    \centering
    \renewcommand{\prtickfontsize}{\tiny}
    \renewcommand{\prlabelfontsize}{\tiny}
    \renewcommand{\prlegendfontsize}{\fontsize{4}{5}\selectfont} 

    \pgfplotsset{
        compat=1.18,
        every axis/.append style={
            pretty line,
            scale only axis,       
            width=0.65\linewidth,   
            height=1.6cm,          
            unbounded coords=discard,
            mark size=0.7pt,       
            clip mode=individual,  
            clip=true,             
            axis on top,           
            xlabel style={font=\prlabelfontsize, at={(ticklabel cs:0.5)}, anchor=north, yshift=2pt},
            ylabel style={font=\prlabelfontsize, at={(ticklabel cs:0.5)}, anchor=south, yshift=-2pt},
            tick label style={font=\prtickfontsize},
            legend style={
                font=\prlegendfontsize, 
                at={(0.98,0.98)}, 
                anchor=north east, 
                cells={anchor=west},
                fill=white, 
                fill opacity=0.8, 
                draw opacity=1,
                row sep=-2.5pt, 
                inner sep=1pt
            }
        },
        bottom x/.style={xlabel={Density $\delta$}, xticklabel style={yshift=0pt}}
    }

    \newcommand{\addstdplot}[5]{
        \addplot+[name path=upper, draw=none, forget plot, opacity=0] table[col sep=comma, x=X, y expr=\thisrow{#3}+\thisrow{#3_std}] {#4};
        \addplot+[name path=lower, draw=none, forget plot, opacity=0] table[col sep=comma, x=X, y expr=\thisrow{#3}-\thisrow{#3_std}] {#4};
        \addplot+[fill opacity=0.15, forget plot, draw=none] fill between[of=upper and lower];
        \addplot+[thick, mark=#2] table[col sep=comma, x=X, y=#3] {#4}; 
        \ifx\&#5\&\else\addlegendentry{#5}\fi
    }

    
    \begin{minipage}{0.24\textwidth}
        \centering
        \begin{tikzpicture}
            \begin{axis}[ylabel={Class. $F_1$}, ymin=0, ymax=1.1, bottom x, legend style={at={(0.98,0.02)}, anchor=south east}]
                \addstdplot{0}{*}{full}{figs/csv_for_figs/sparsity_ablations/sparsity_f1.csv}{\ourmethod\!($\mathcal{G}^*$)}
            \end{axis}
        \end{tikzpicture}
    \end{minipage}\hfill
    \begin{minipage}{0.24\textwidth}
        \centering
        \begin{tikzpicture}
            \begin{axis}[ylabel={Subset $F_1$}, ymin=0, ymax=1.1, bottom x, legend style={at={(0.98,0.02)}, anchor=south east}]
                \addstdplot{0}{*}{full}{figs/csv_for_figs/sparsity_ablations/sparsity_subset.csv}{\ourmethod\!($\mathcal{G}^*$)}
            \end{axis}
        \end{tikzpicture}
    \end{minipage}\hfill
    \begin{minipage}{0.24\textwidth}
        \centering
        \begin{tikzpicture}
            \begin{axis}[ylabel={AMI}, ymin=0, ymax=1.1, bottom x]
                \def\file{figs/csv_for_figs/sparsity_ablations/sparsity_ami_full.csv}
                \addstdplot{1}{*}{none}{\file}{Unbiased}
                \addstdplot{2}{square*}{confounder}{\file}{Conf.}
                \addstdplot{3}{triangle*}{collider}{\file}{Sel.}
            \end{axis}
        \end{tikzpicture}
    \end{minipage}\hfill
    \begin{minipage}{0.24\textwidth}
        \centering
        \begin{tikzpicture}
            \begin{axis}[ylabel={Coverage}, ymin=0, ymax=1.1, bottom x]
                \def\file{figs/csv_for_figs/sparsity_ablations/sparsity_coverage_full.csv}
                \addstdplot{1}{*}{none}{\file}{Unbiased}
                \addstdplot{2}{square*}{confounder}{\file}{Conf.}
                \addstdplot{3}{triangle*}{collider}{\file}{Sel.}
            \end{axis}
        \end{tikzpicture}
    \end{minipage}

    \caption{Ablation over the fraction of contexts in which variables undergo causal mechanism shifts. \ourmethod\!($\mathcal{G}^*$) works well even when the assumption is relaxed.}
    \label{fig:sparsity_ablation}
\end{figure*}

%% file: figs/tex_figs/contexts_new.tex
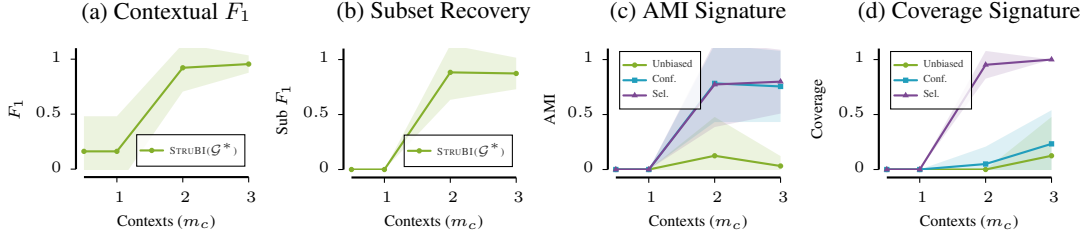
\begin{figure*}[t]
    \centering
    \renewcommand{\prtickfontsize}{\tiny}
    \renewcommand{\prlabelfontsize}{\tiny}
    \renewcommand{\prlegendfontsize}{\fontsize{4}{5}\selectfont} 

    \pgfplotsset{
        compat=1.18,
        every axis/.append style={
            pretty line,
            scale only axis,       
            width=0.65\linewidth,   
            height=1.6cm,          
            unbounded coords=discard,
            mark size=0.6pt,       
            clip mode=individual,  
            clip=true,             
            axis on top,           
            xlabel style={font=\prlabelfontsize, at={(ticklabel cs:0.5)}, anchor=north, yshift=2pt},
            ylabel style={font=\prlabelfontsize, at={(ticklabel cs:0.5)}, anchor=south, yshift=-2pt},
            tick label style={font=\prtickfontsize},
            legend style={
                font=\prlegendfontsize, 
                at={(0.98,0.98)}, 
                anchor=north east, 
                cells={anchor=west},
                fill=white, 
                fill opacity=0.8, 
                draw opacity=1,
                row sep=-2.5pt, 
                inner sep=1pt
            }
        },
        bottom x/.style={xlabel={Contexts ($m_c$)}, xticklabel style={yshift=0pt}}
    }

    \newcommand{\addstdplot}[5]{
        \addplot+[name path=upper, draw=none, forget plot, opacity=0] table[col sep=comma, x=X, y expr=\thisrow{#3}+\thisrow{#3_std}] {#4};
        \addplot+[name path=lower, draw=none, forget plot, opacity=0] table[col sep=comma, x=X, y expr=\thisrow{#3}-\thisrow{#3_std}] {#4};
        \addplot+[fill opacity=0.15, forget plot, draw=none] fill between[of=upper and lower];
        \addplot+[thick, mark=#2] table[col sep=comma, x=X, y=#3] {#4}; 
        \ifx\&#5\&\else\addlegendentry{#5}\fi
    }

    
    \begin{minipage}{0.24\textwidth}
        \centering
        \begin{tikzpicture}
            \begin{axis}[
                title={\small (a) Contextual $F_1$}, ylabel={$F_1$}, ymin=0, ymax=1.1, bottom x,
                legend style={at={(0.98,0.02)}, anchor=south east} 
            ]
                \addstdplot{0}{*}{5}{figs/csv_for_figs/contexts_ablations/contexts_f1.csv}{\ourmethod\!($\mathcal{G}^*$)}
            \end{axis}
        \end{tikzpicture}
    \end{minipage}\hfill
    \begin{minipage}{0.24\textwidth}
        \centering
        \begin{tikzpicture}
            \begin{axis}[
                title={\small (b) Subset Recovery}, ylabel={Sub $F_1$}, ymin=0, ymax=1.1, bottom x,
                legend style={at={(0.98,0.02)}, anchor=south east} 
            ]
                \addstdplot{0}{*}{5}{figs/csv_for_figs/contexts_ablations/contexts_subset.csv}{\ourmethod\!($\mathcal{G}^*$)}
            \end{axis}
        \end{tikzpicture}
    \end{minipage}\hfill
    \begin{minipage}{0.24\textwidth}
        \centering
        \begin{tikzpicture}
            \begin{axis}[title={\small (c) AMI Signature}, ylabel={AMI}, ymin=0, ymax=1.1, bottom x,
                legend style={at={(0.02,0.98)}, anchor=north west}]
                \def\file{figs/csv_for_figs/contexts_ablations/contexts_ami_5.csv}
                \addstdplot{1}{*}{none}{\file}{Unbiased}
                \addstdplot{2}{square*}{confounder}{\file}{Conf.}
                \addstdplot{3}{triangle*}{collider}{\file}{Sel.}
            \end{axis}
        \end{tikzpicture}
    \end{minipage}\hfill
    \begin{minipage}{0.24\textwidth}
        \centering
        \begin{tikzpicture}
            \begin{axis}[title={\small (d) Coverage Signature}, ylabel={Coverage}, ymin=0, ymax=1.1, bottom x,
                legend style={at={(0.02,0.98)}, anchor=north west}]
                \def\file{figs/csv_for_figs/contexts_ablations/contexts_coverage_5.csv}
                \addstdplot{1}{*}{none}{\file}{Unbiased}
                \addstdplot{2}{square*}{confounder}{\file}{Conf.}
                \addstdplot{3}{triangle*}{collider}{\file}{Sel.}
            \end{axis}
        \end{tikzpicture}
    \end{minipage}

    \caption{Ablation over the ratio of the number of observed variables with the total number of contexts. The X-axis denotes the factor multiplied to the DAG size to obtain the contexts.}
    \label{fig:contexts_ablation}
\end{figure*}

%% file: figs/tex_figs/perturbed.tex
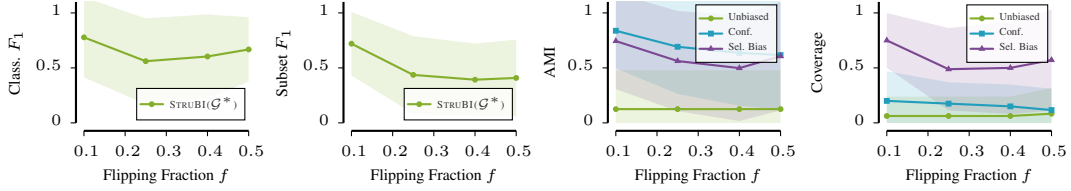
\begin{figure*}[]
    \centering
    \renewcommand{\prtickfontsize}{\tiny}
    \renewcommand{\prlabelfontsize}{\tiny}
    \renewcommand{\prlegendfontsize}{\fontsize{4}{5}\selectfont} 

    \pgfplotsset{
        compat=1.18,
        every axis/.append style={
            pretty line,        
            scale only axis,
            width=0.65\linewidth, 
            height=1.6cm,      
            unbounded coords=discard,
            mark size=0.7pt,
            restrict x to domain=0.1:0.5, 
            xmin=0.1, xmax=0.5, 
            clip mode=individual,
            clip=true,
            axis on top,
            xtick={0.1, 0.2, 0.3, 0.4, 0.5},
            xlabel style={font=\prlabelfontsize, at={(ticklabel cs:0.5)}, anchor=north, yshift=2pt},
            ylabel style={font=\prlabelfontsize, at={(ticklabel cs:0.5)}, anchor=south, yshift=-2pt},
            tick label style={font=\prtickfontsize},
            legend style={
                font=\prlegendfontsize, 
                at={(0.98,0.98)}, 
                anchor=north east, 
                cells={anchor=west},
                fill=white, 
                fill opacity=0.8, 
                draw opacity=1,
                row sep=-2.5pt, 
                inner sep=1pt
            }
        },
        bottom x/.style={xlabel={Flipping Fraction $f$}}
    }

    \newcommand{\addstdplot}[5]{
        \addplot+[name path=upper, draw=none, forget plot, opacity=0] table[col sep=comma, x=X, y expr=\thisrow{#3}+\thisrow{#3_std}] {#4};
        \addplot+[name path=lower, draw=none, forget plot, opacity=0] table[col sep=comma, x=X, y expr=\thisrow{#3}-\thisrow{#3_std}] {#4};
        \addplot+[fill opacity=0.15, forget plot, draw=none] fill between[of=upper and lower];
        \addplot+[thick, mark=#2] table[col sep=comma, x=X, y=#3] {#4}; 
        \ifx\&#5\&\else\addlegendentry{#5}\fi
    }

    
    \begin{minipage}{0.24\textwidth}
        \centering
        \begin{tikzpicture}
            \begin{axis}[
                ylabel={Class. $F_1$}, ymin=0, ymax=1.1, bottom x,
                legend style={at={(0.98,0.02)}, anchor=south east} 
            ]
                \addstdplot{0}{*}{perturbed}{figs/csv_for_figs/perturbed_graph_ablations/perturbed_graph_f1.csv}{\ourmethod\!($\mathcal{G}^*$)}
            \end{axis}
        \end{tikzpicture}
    \end{minipage}\hfill
    \begin{minipage}{0.24\textwidth}
        \centering
        \begin{tikzpicture}
            \begin{axis}[
                ylabel={Subset $F_1$}, ymin=0, ymax=1.1, bottom x,
                legend style={at={(0.98,0.02)}, anchor=south east} 
            ]
                \addstdplot{0}{*}{perturbed}{figs/csv_for_figs/perturbed_graph_ablations/perturbed_graph_subset.csv}{\ourmethod\!($\mathcal{G}^*$)}
            \end{axis}
        \end{tikzpicture}
    \end{minipage}\hfill
    \begin{minipage}{0.24\textwidth}
        \centering
        \begin{tikzpicture}
            \begin{axis}[ylabel={AMI}, ymin=0, ymax=1.1, bottom x]
                \def\file{figs/csv_for_figs/perturbed_graph_ablations/perturbed_graph_ami_perturbed.csv}
                \addstdplot{1}{*}{none}{\file}{Unbiased}
                \addstdplot{2}{square*}{confounder}{\file}{Conf.}
                \addstdplot{3}{triangle*}{collider}{\file}{Sel. Bias}
            \end{axis}
        \end{tikzpicture}
    \end{minipage}\hfill
    \begin{minipage}{0.24\textwidth}
        \centering
        \begin{tikzpicture}
            \begin{axis}[ylabel={Coverage}, ymin=0, ymax=1.1, bottom x]
                \def\file{figs/csv_for_figs/perturbed_graph_ablations/perturbed_graph_coverage_perturbed.csv}
                \addstdplot{1}{*}{none}{\file}{Unbiased}
                \addstdplot{2}{square*}{confounder}{\file}{Conf.}
                \addstdplot{3}{triangle*}{collider}{\file}{Sel. Bias}
            \end{axis}
        \end{tikzpicture}
    \end{minipage}
    
    \caption{Ablation of performance mettrics across flipping fractions. Since the performance drop is not huge, we can say  that \ourmethod can accommodate graph mis-specification.}
    \label{fig:perturbed_graph_ablation}
\end{figure*}

%% file: figs/tex_figs/main_ablation_extras.tex
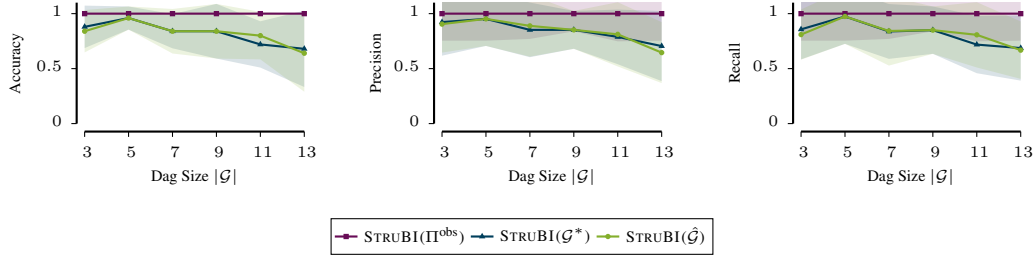
\begin{figure*}[t] 
    \centering
    \renewcommand{\prtickfontsize}{\tiny}
    \renewcommand{\prlabelfontsize}{\tiny}
    \renewcommand{\prlegendfontsize}{\tiny}

    \pgfplotsset{
        compat=1.18,
        every axis/.append style={
            pretty line,        
            scale only axis,
            width=0.65\linewidth, 
            height=1.6cm,
            ymin=0, ymax=1.1,
            restrict x to domain=3:13, xmin=3, xmax=13, 
            clip mode=individual, clip=true, axis on top,
            xtick={3, 5, 7, 9, 11, 13},
            xlabel={Dag Size $\lvert\mathcal{G}\rvert$},
            xlabel style={font=\prlabelfontsize, at={(ticklabel cs:0.5)}, anchor=north, yshift=2pt},
            ylabel style={font=\prlabelfontsize, at={(ticklabel cs:0.5)}, anchor=south, yshift=-2pt},
            tick label style={font=\prtickfontsize}
        }
    }

    \newcommand{\addstdplot}[5]{
        \addplot[name path=upper, draw=none, forget plot, opacity=0] table[col sep=comma, x=X, y expr=\thisrow{#3}+\thisrow{#3_std}] {#4};
        \addplot[name path=lower, draw=none, forget plot, opacity=0] table[col sep=comma, x=X, y expr=\thisrow{#3}-\thisrow{#3_std}] {#4};
        \addplot[fill=#1, fill opacity=0.15, forget plot, draw=none] fill between[of=upper and lower];
        \addplot[thick, color=#1, mark=#2, mark size=0.6pt] table[col sep=comma, x=X, y=#3] {#4}; 
        \ifx\&#5\&\else\addlegendentry{#5}\fi
    }

    \def\colPartitions{pr-color-purple} 
    \def\colTOPIC{pr-color1a}       
    \def\colFull{pr-color3a}      

    \begin{minipage}{0.32\textwidth}
        \centering
        \begin{tikzpicture}
            \begin{axis}[ylabel={Accuracy}, legend columns=-1, legend to name=LegPerformance]
                \def\f{figs/csv_for_figs/main_ablations/main_accuracy.csv}
                \addstdplot{\colPartitions}{square*}{partitions}{\f}{\ourmethod\!($\Pi^{\text{obs}}$)}
                \addstdplot{\colFull}{triangle*}{full}{\f}{\ourmethod\!($\mathcal{G}^*$)}
                \addstdplot{\colTOPIC}{*}{TOPIC}{\f}{\ourmethod\!($\hat{\mathcal{G}}$)}
            \end{axis}
        \end{tikzpicture}
    \end{minipage}\hfill
    \begin{minipage}{0.32\textwidth}
        \centering
        \begin{tikzpicture}
            \begin{axis}[ylabel={Precision}]
                \def\f{figs/csv_for_figs/main_ablations/main_precision.csv}
                \addstdplot{\colPartitions}{square*}{partitions}{\f}{}
                \addstdplot{\colFull}{triangle*}{full}{\f}{}
                \addstdplot{\colTOPIC}{*}{TOPIC}{\f}{}
            \end{axis}
        \end{tikzpicture}
    \end{minipage}\hfill
    \begin{minipage}{0.32\textwidth}
        \centering
        \begin{tikzpicture}
            \begin{axis}[ylabel={Recall}]
                \def\f{figs/csv_for_figs/main_ablations/main_recall.csv}
                \addstdplot{\colPartitions}{square*}{partitions}{\f}{}
                \addstdplot{\colFull}{triangle*}{full}{\f}{}
                \addstdplot{\colTOPIC}{*}{TOPIC}{\f}{}
            \end{axis}
        \end{tikzpicture}
    \end{minipage}

    \vspace{0.8em}
    \centerline{\ref{LegPerformance}} 
    \caption{The accuracy, precision, and recall for \ourmethod's oracles for the task of structural bias classification}
    \label{fig:oracles_extra}
\end{figure*}

%% file: figs/tex_figs/baselines_new_fill.tex
\begin{figure*}[h] 
    \centering
    \renewcommand{\prtickfontsize}{\tiny}
    \renewcommand{\prlabelfontsize}{\tiny}
    \renewcommand{\prlegendfontsize}{\tiny}

    \pgfplotsset{
        every axis/.append style={
            pretty line,
            scale only axis,       
            width=0.6\linewidth,  
            height=1.3cm,
            unbounded coords=discard,
            xmin=3, xmax=13, 
            enlargelimits=false,
            xtick={3, 5, 7, 9, 11, 13},
            xlabel={Dag Size $|\mathcal{G}|$},
            xlabel style={at={(ticklabel cs:0.5)}, anchor=north, yshift=-4pt},
            y label style={font=\prlabelfontsize},
            x label style={font=\prlabelfontsize},
            clip=true,
            restrict x to domain=3:15,
            mark size=0.7pt,       
            title={}, 
        }
    }

    \newcommand{\addbaselineplot}[5]{
        \addplot[#1, name path=U, draw=none, forget plot] table[x=X, y expr=\thisrow{#3}+\thisrow{#3_std}, col sep=comma] {#4};
        \addplot[#1, name path=D, draw=none, forget plot] table[x=X, y expr=\thisrow{#3}-\thisrow{#3_std}, col sep=comma] {#4};
        \addplot[#1, fill opacity=0.1, forget plot] fill between [of=U and D];
        \addplot[thick, #1, mark=#2] table[x=X, y=#3, col sep=comma] {#4};
        \ifx\&#5\&\else\addlegendentry{#5}\fi
    }

    \def\colFull{pr-color1a}       
    \def\colCoCo{pr-color1c}       
    \def\colFCIJ{pr-color2a}       
    \def\colFCIP{pr-color1b}       
    \def\colLSP{pr-color3a}        
    \def\colLSC{pr-color-gray4}    

    \begin{minipage}{0.32\textwidth}
        \centering
        \begin{tikzpicture}
            \begin{axis}[title = {\tiny Latent Discovery}, ylabel={$F_1$}, ymin=0, ymax=1.1]
                \def\f{figs/csv_for_figs/baseline_plots/baseline_f1_nc.csv}
                \addbaselineplot{\colFull}{*}{topic}{\f}{}
                \addbaselineplot{\colCoCo}{square*}{coco}{\f}{}
                \addbaselineplot{\colFCIJ}{triangle*}{fci-jci}{\f}{}
                \addbaselineplot{\colFCIP}{diamond*}{fci-pooled}{\f}{}
                \addbaselineplot{\colLSP}{*}{dai_pooled}{\f}{}
                \addbaselineplot{\colLSC}{x}{dai_context}{\f}{}
            \end{axis}
        \end{tikzpicture}
    \end{minipage}\hfill
    \begin{minipage}{0.32\textwidth}
        \centering
        \begin{tikzpicture}
            \begin{axis}[title = {\tiny Subset Recovery}, ylabel={$F_1$}, ymin=0, ymax=1.1]
                \def\f{figs/csv_for_figs/baseline_plots/baseline_subset.csv}
                \addbaselineplot{\colFull}{*}{topic}{\f}{}
                \addbaselineplot{\colCoCo}{square*}{coco}{\f}{}
                \addbaselineplot{\colFCIJ}{triangle*}{fci-jci}{\f}{}
                \addbaselineplot{\colFCIP}{diamond*}{fci-pooled}{\f}{}
                \addbaselineplot{\colLSP}{*}{dai_pooled}{\f}{}
                \addbaselineplot{\colLSC}{x}{dai_context}{\f}{}
            \end{axis}
        \end{tikzpicture}
    \end{minipage}\hfill
    \begin{minipage}{0.32\textwidth}
        \centering
        \begin{tikzpicture}
            \begin{axis}[title = {\tiny Structural Bias Classification}, ylabel={$F_1$}, ymin=0, ymax=1.1,
                legend columns=-1, legend to name=LegBaselines]
                \def\f{figs/csv_for_figs/baseline_plots/baseline_f1.csv}
                \addbaselineplot{\colFull}{*}{topic}{\f}{\ourmethod\!($\hat{\mathcal{G}}$)}
                \addbaselineplot{\colCoCo}{square*}{coco}{\f}{CoCo($\mathcal{G}^*$)}
                \addbaselineplot{\colFCIJ}{triangle*}{fci-jci}{\f}{FCI-JCI}
                \addbaselineplot{\colFCIP}{diamond*}{fci-pooled}{\f}{FCI-p}
                \addbaselineplot{\colLSP}{*}{dai_pooled}{\f}{LS-P}
                \addbaselineplot{\colLSC}{x}{dai_context}{\f}{LS-C}
            \end{axis}
        \end{tikzpicture}
    \end{minipage}

    \centerline{\ref{LegBaselines}}
    \caption{F1 scores on (a) latent variable discovery, (b) affected subset recovery, and (c) structural bias classification. This figure also shows the error bars, which were not included in Figure \ref{fig:baselines} for clarity.}
    \label{fig:baselines_errors}
\end{figure*}
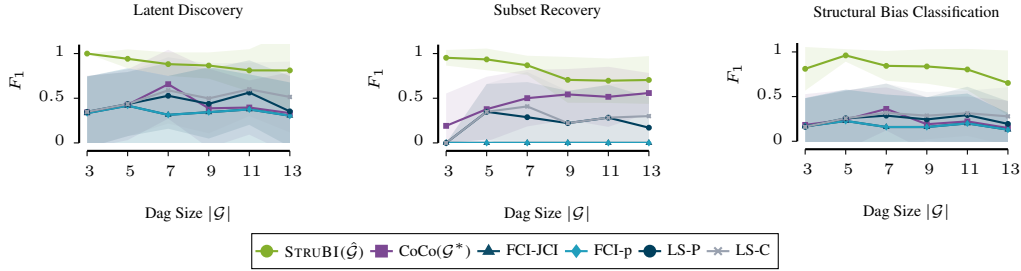

%% file: figs/tex_figs/baselines_new_extras.tex
\begin{figure*}[h]
    \centering
    \renewcommand{\prtickfontsize}{\tiny}
    \renewcommand{\prlabelfontsize}{\tiny}
    \renewcommand{\prlegendfontsize}{\tiny}

    \pgfplotsset{
        compat=1.18,
        every axis/.append style={
            pretty line,        
            scale only axis,
            width=0.65\linewidth, 
            height=1.3cm,       
            unbounded coords=discard,
            ymin=0, ymax=1.1,
            xmin=3, xmax=13, 
            clip mode=individual,
            clip=true,
            axis on top,
            xtick={3, 5, 7, 9, 11, 13},
            xlabel={Dag Size $\lvert\mathcal{G}\rvert$},
            mark size = 0.7pt,
            xlabel style={font=\prlabelfontsize, at={(ticklabel cs:0.5)}, anchor=north, yshift=2pt},
            ylabel style={font=\prlabelfontsize, at={(ticklabel cs:0.5)}, anchor=south, yshift=-2pt},
            tick label style={font=\prtickfontsize}
        }
    }

    \def\colFull{pr-color1a}       
    \def\colCoCo{pr-color1c}       
    \def\colFCIJ{pr-color2a}       
    \def\colFCIP{pr-color1b}       
    \def\colLSP{pr-color3a}        
    \def\colLSC{pr-color-gray4}    

    \newcommand{\addstdplot}[5]{
        \addplot[name path=upper, draw=none, forget plot, opacity=0] table[col sep=comma, x=X, y expr=\thisrow{#3}+\thisrow{#3_std}] {#4};
        \addplot[name path=lower, draw=none, forget plot, opacity=0] table[col sep=comma, x=X, y expr=\thisrow{#3}-\thisrow{#3_std}] {#4};
        \addplot[fill=#1, fill opacity=0.15, forget plot, draw=none] fill between[of=upper and lower];
        \addplot[thick, color=#1, mark=#2, mark size=0.5pt] table[col sep=comma, x=X, y=#3] {#4}; 
        \ifx\&#5\&\else\addlegendentry{#5}\fi
    }

    \begin{minipage}{0.32\textwidth}
        \centering
        \begin{tikzpicture}
            \begin{axis}[ylabel={Accuracy},
                legend columns=-1, legend to name=LegBaselines]
                \def\f{figs/csv_for_figs/baseline_plots/baseline_accuracy.csv}
                \addstdplot{\colFull}{*}{topic}{\f}{\ourmethod\!($\hat{\mathcal{G}}$)}
                \addstdplot{\colCoCo}{square*}{coco}{\f}{CoCo($\mathcal{G}^*$)}
                \addstdplot{\colFCIJ}{*}{fci-jci}{\f}{FCI-JCI}
                \addstdplot{\colFCIP}{triangle*}{fci-pooled}{\f}{FCI-P}
                \addstdplot{\colLSP}{*}{dai_pooled}{\f}{LS-P}
                \addstdplot{\colLSC}{square*}{dai_context}{\f}{LS-C}
            \end{axis}
        \end{tikzpicture}
    \end{minipage}\hfill
    \begin{minipage}{0.32\textwidth}
        \centering
        \begin{tikzpicture}
            \begin{axis}[ylabel={Precision}]
                \def\f{figs/csv_for_figs/baseline_plots/baseline_precision.csv}
                \addstdplot{\colFull}{*}{topic}{\f}{}
                \addstdplot{\colCoCo}{square*}{coco}{\f}{}
                \addstdplot{\colFCIJ}{*}{fci-jci}{\f}{}
                \addstdplot{\colFCIP}{triangle*}{fci-pooled}{\f}{}
                \addstdplot{\colLSP}{*}{dai_pooled}{\f}{}
                \addstdplot{\colLSC}{square*}{dai_context}{\f}{}
            \end{axis}
        \end{tikzpicture}
    \end{minipage}\hfill
    \begin{minipage}{0.32\textwidth}
        \centering
        \begin{tikzpicture}
            \begin{axis}[ylabel={Recall}]
                \def\f{figs/csv_for_figs/baseline_plots/baseline_recall.csv}
                \addstdplot{\colFull}{*}{topic}{\f}{}
                \addstdplot{\colCoCo}{square*}{coco}{\f}{}
                \addstdplot{\colFCIJ}{*}{fci-jci}{\f}{}
                \addstdplot{\colFCIP}{triangle*}{fci-pooled}{\f}{}
                \addstdplot{\colLSP}{*}{dai_pooled}{\f}{}
                \addstdplot{\colLSC}{square*}{dai_context}{\f}{}
            \end{axis}
        \end{tikzpicture}
    \end{minipage}

    \vspace{0.8em}
    \centerline{\ref{LegBaselines}}

    \caption{\ourmethod outperforms the baselines for structural bias classification in terms of accuracy, precision, and recall as well.}
    \label{fig:baseline_metrics}
\end{figure*}
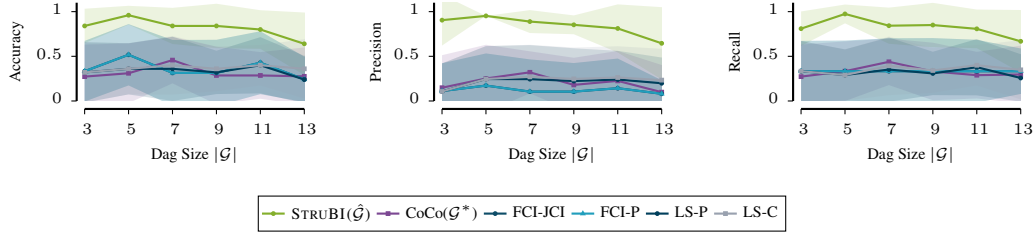

%% file: figs/tex_figs/latent_selection_vs_topic.tex
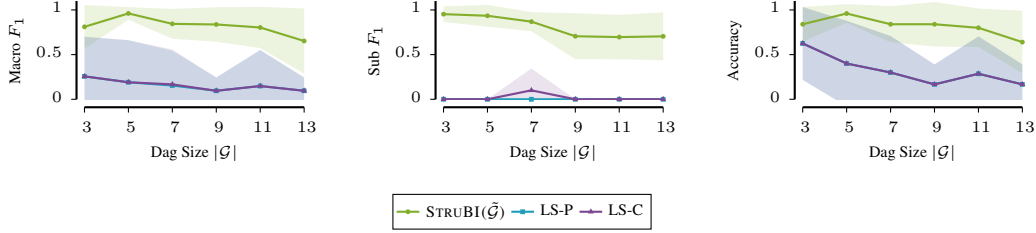
\begin{figure}[h!]
    \centering
    \renewcommand{\prtickfontsize}{\tiny}
    \renewcommand{\prlabelfontsize}{\tiny}
    \renewcommand{\prlegendfontsize}{\tiny}

    \pgfplotsset{
        compat=1.18,
        every axis/.append style={
            pretty line,        
            scale only axis,
            width=0.65\linewidth, 
            height=1.3cm,       
            unbounded coords=discard,
            ymin=0, ymax=1.1,
            restrict x to domain=3:13, 
            xmin=3, xmax=13, 
            clip mode=individual,
            clip=true,
            axis on top,
            xtick={3, 5, 7, 9, 11, 13},
            mark size=0.7pt,
            xlabel={Dag Size $\lvert \mathcal{G} \rvert$},
            xlabel style={font=\prlabelfontsize, at={(ticklabel cs:0.5)}, anchor=north, yshift=2pt},
            ylabel style={font=\prlabelfontsize, at={(ticklabel cs:0.5)}, anchor=south, yshift=-2pt},
            tick label style={font=\prtickfontsize}
        }
    }

    \newcommand{\addstdplot}[5]{
        \addplot+[name path=upper, draw=none, forget plot, opacity=0] table[col sep=comma, x=X, y expr=\thisrow{#3}+\thisrow{#3_std}] {#4};
        \addplot+[name path=lower, draw=none, forget plot, opacity=0] table[col sep=comma, x=X, y expr=\thisrow{#3}-\thisrow{#3_std}] {#4};
        \addplot+[fill opacity=0.15, forget plot, draw=none] fill between[of=upper and lower];
        \addplot+[thick, mark=#2, mark size=0.5pt] table[col sep=comma, x=X, y=#3] {#4}; 
        \ifx\&#5\&\else\addlegendentry{#5}\fi
    }

    \begin{minipage}{0.32\textwidth}
        \centering
        \begin{tikzpicture}
            \begin{axis}[ ylabel={Macro $F_1$},
                legend columns=-1, 
                legend to name=LegLatent]
                
                \def\f{figs/csv_for_figs/latent_comparison/latent_f1.csv}
                \addstdplot{1}{*}{topic}{\f}{\ourmethod\!($\tilde{\mathcal{G}}$)}
                \addstdplot{2}{square*}{dai_pooled}{\f}{LS-P}
                \addstdplot{3}{triangle*}{dai_context}{\f}{LS-C}
            \end{axis}
        \end{tikzpicture}
    \end{minipage}\hfill
    \begin{minipage}{0.32\textwidth}
        \centering
        \begin{tikzpicture}
            \begin{axis}[ ylabel={Sub $F_1$}]
                \def\f{figs/csv_for_figs/latent_comparison/latent_subset.csv}
                \addstdplot{1}{*}{topic}{\f}{}
                \addstdplot{2}{square*}{dai_pooled}{\f}{}
                \addstdplot{3}{triangle*}{dai_context}{\f}{}
            \end{axis}
        \end{tikzpicture}
    \end{minipage}\hfill
    \begin{minipage}{0.32\textwidth}
        \centering
        \begin{tikzpicture}
            \begin{axis}[ ylabel={Accuracy}]
                \def\f{figs/csv_for_figs/latent_comparison/latent_accuracy.csv}
                \addstdplot{1}{*}{topic}{\f}{}
                \addstdplot{2}{square*}{dai_pooled}{\f}{}
                \addstdplot{3}{triangle*}{dai_context}{\f}{}
            \end{axis}
        \end{tikzpicture}
    \end{minipage}

    \vspace{0.8em}
    \centerline{\ref{LegLatent}}

    \caption{Our \topic oracle outperforms \textit{LS-Pooled} and \textit{LS-Contexts} on linear data with Gaussian and uniform noise sampled with equal probability. These are the primary paramteric assumptions by \cite{dai2025latent}.}
    \label{fig:latent_selection_comparison}
\end{figure}

%% file: figs/tex_figs/runtime_baselines.tex
\begin{figure*}[h!]
    \centering
    \renewcommand{\prtickfontsize}{\tiny}
    \renewcommand{\prlabelfontsize}{\tiny}
    \renewcommand{\prlegendfontsize}{\tiny}

    \begin{tikzpicture}
        \begin{axis}[
            pretty line,
            width=0.35\linewidth, 
            height=3.5cm,
            xlabel={Dag Size $\lvert\mathcal{G}\rvert$},
            ylabel={Time (s)},
            ymin=0.000, ymax=400,
            unbounded coords=discard,
            clip=true,
            xtick={3, 5, 7, 9, 11, 13},
            xmax=13,
            mark size = 0.7pt,
            legend columns=2, 
            legend pos=outer north east,
            legend cell align={left},
            legend style={font=\prlegendfontsize}
        ]
            \def\f{figs/csv_for_figs/baseline_plots/baseline_runtime.csv}
            \def\colFull{pr-color1a}       
            \def\colCoCo{pr-color1c}       
            \def\colFCIJ{pr-color2a}       
            \def\colFCIP{pr-color1b}       
            \def\colLSP{pr-color3a}        
            \def\colLSC{pr-color-gray4}    

            \addplot [name path=U1, draw=none, forget plot] table[x=X, y expr=\thisrow{topic}+\thisrow{topic_std}, col sep=comma] {\f};
            \addplot [name path=D1, draw=none, forget plot] table[x=X, y expr=\thisrow{topic}-\thisrow{topic_std}, col sep=comma] {\f};
            \addplot [fill=\colFull, fill opacity=0.0, forget plot] fill between [of=U1 and D1];
            \addplot [thick, \colFull, mark=*] table[x=X, y=topic, col sep=comma] {\f}; 
            \addlegendentry{\ourmethod\!($\mathcal{G}^*$)}

            \addplot [name path=U2, draw=none, forget plot] table[x=X, y expr=\thisrow{coco}+\thisrow{coco_std}, col sep=comma] {\f};
            \addplot [name path=D2, draw=none, forget plot] table[x=X, y expr=\thisrow{coco}-\thisrow{coco_std}, col sep=comma] {\f};
            \addplot [fill=\colCoCo, fill opacity=0.0, forget plot] fill between [of=U2 and D2];
            \addplot [thick, \colCoCo, mark=square*] table[x=X, y=coco, col sep=comma] {\f};
            \addlegendentry{CoCo($\mathcal{G}^*$)}
            
            \addplot [name path=U3, draw=none, forget plot] table[x=X, y expr=\thisrow{fci-jci}+\thisrow{fci-jci_std}, col sep=comma] {\f};
            \addplot [name path=D3, draw=none, forget plot] table[x=X, y expr=\thisrow{fci-jci}-\thisrow{fci-jci_std}, col sep=comma] {\f};
            \addplot [fill=\colFCIJ, fill opacity=0.0, forget plot] fill between [of=U2 and D2];
            \addplot [thick, \colFCIJ, mark=square*] table[x=X, y=fci-jci, col sep=comma] {\f};
            \addlegendentry{FCI-JCI}

            \addplot [name path=U4, draw=none, forget plot] table[x=X, y expr=\thisrow{fci-pooled}+\thisrow{fci-pooled_std}, col sep=comma] {\f};
            \addplot [name path=D4, draw=none, forget plot] table[x=X, y expr=\thisrow{fci-pooled}-\thisrow{fci-pooled_std}, col sep=comma] {\f};
            \addplot [fill=\colFCIP, fill opacity=0.0, forget plot] fill between [of=U2 and D2];
            \addplot [thick, \colFCIP, mark=square*] table[x=X, y=fci-pooled, col sep=comma] {\f};
            \addlegendentry{FCI-P}

            \addplot [name path=U5, draw=none, forget plot] table[x=X, y expr=\thisrow{dai_context}+\thisrow{dai_context_std}, col sep=comma] {\f};
            \addplot [name path=D5, draw=none, forget plot] table[x=X, y expr=\thisrow{dai_context}-\thisrow{dai_context_std}, col sep=comma] {\f};
            \addplot [fill=\colLSP, fill opacity=0.0, forget plot] fill between [of=U3 and D3];
            \addplot [thick, \colLSP, mark=triangle*] table[x=X, y=dai_context, col sep=comma] {\f};
            \addlegendentry{LS-C}

            \addplot [name path=U6, draw=none, forget plot] table[x=X, y expr=\thisrow{dai_pooled}+\thisrow{dai_pooled_std}, col sep=comma] {\f};
            \addplot [name path=D6, draw=none, forget plot] table[x=X, y expr=\thisrow{dai_pooled}-\thisrow{dai_pooled_std}, col sep=comma] {\f};
            \addplot [fill=\colLSC, fill opacity=0.0, forget plot] fill between [of=U4 and D4];
            \addplot [thick, \colLSC, mark=diamond*] table[x=X, y=dai_pooled, col sep=comma] {\f};
            \addlegendentry{LS-P}

        \end{axis}
    \end{tikzpicture}
    \caption{Runtime Analysis for all baselines across different graph sizes.}
    \label{fig:runtime_comparison}
\end{figure*}
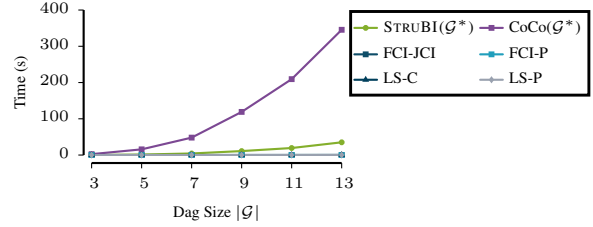

%% file: figs/tex_figs/sachs_coco.tex
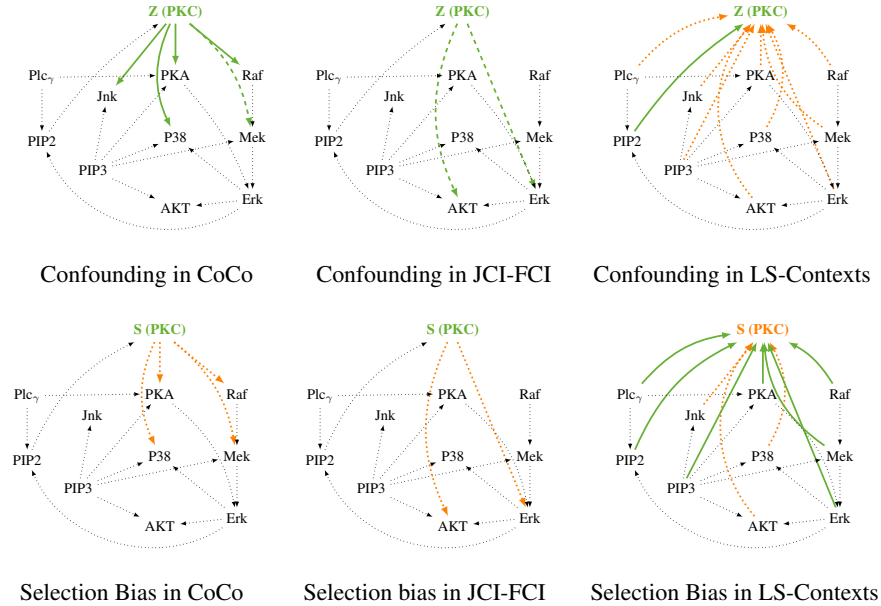
\begin{figure}[h!]
  \centering
  \begin{tabular}[b]{c}
    \scalebox{0.55}{
    \begin{tikzpicture}
      \node [green(ryb)] (pkc)  {\textbf{Z (PKC)}};
      \node [below left = 1cm and 2cm of pkc] (plcg) {Plc${}_{\gamma}$};
      \node [below = 1cm of plcg] (pip2) {PIP2};
      \node [below right = 0.2cm and 0.3cm of pip2] (pip3) {PIP3};
      \node [below = 1cm of pkc] (pka) {PKA};
      \node [right = 1cm of pka] (raf) {Raf};
      \node [below = 1cm of raf] (mek) {Mek};
      \node [below = 1cm of mek] (erk) {Erk};
      \node [below = 1cm of pka] (p38) {P38};
      \node [below = 2.7cm of pka] (akt) {AKT};
      \node [below left = 1.5cm and 0.5cm of pkc] (jnk) {Jnk};

      \path[-{latex[]}, line width=1.2pt, green(ryb)]
      (pkc) edge (pka)
      (pkc) edge (jnk)
      (pkc) edge (raf)
      (pkc) edge[bend right=25] (p38);
      (pkc) edge[bend right=25] (pip3);

      \path[-{latex[]}, line width=1.2pt, dashed, green(ryb)]
      (pkc) edge[bend left=20] (mek);

      \path[-{latex[]}, dotted, thin]
      (plcg) edge (pip2) (plcg) edge (pka)
      (pip2) edge[bend left=10] (pkc)
      (pip3) edge (jnk) (pip3) edge (pka) (pip3) edge (p38) (pip3) edge (mek) (pip3) edge (akt)
      (pka) edge[bend left=10] (erk) (raf) edge (mek) (mek) edge (erk)
      (erk) edge (akt) (erk) edge (p38) (erk) edge[bend left=50] (pip2);
    \end{tikzpicture}
    }\\
  \small Confounding in CoCo
  \end{tabular}%
  \begin{tabular}[b]{c}
    \scalebox{0.55}{
    \begin{tikzpicture}
      \node [green(ryb)] (pkc)  {\textbf{Z (PKC)}};
      \node [below left = 1cm and 2cm of pkc] (plcg) {Plc${}_{\gamma}$};
      \node [below = 1cm of plcg] (pip2) {PIP2};
      \node [below right = 0.2cm and 0.3cm of pip2] (pip3) {PIP3};
      \node [below = 1cm of pkc] (pka) {PKA};
      \node [right = 1cm of pka] (raf) {Raf};
      \node [below = 1cm of raf] (mek) {Mek};
      \node [below = 1cm of mek] (erk) {Erk};
      \node [below = 1cm of pka] (p38) {P38};
      \node [below = 2.7cm of pka] (akt) {AKT};
      \node [below left = 1.5cm and 0.5cm of pkc] (jnk) {Jnk};

      \path[-{latex[]}, line width=1.2pt, dashed, green(ryb)]
      (pkc) edge (erk)
      (pkc) edge[bend right=25] (akt);
      
      \path[-{latex[]}, dotted, thin]
      (plcg) edge (pip2) (plcg) edge (pka)
      (pip2) edge[bend left=10] (pkc)
      (pip3) edge (jnk) (pip3) edge (pka) (pip3) edge (p38) (pip3) edge (mek) (pip3) edge (akt)
      (pka) edge[bend left=10] (erk) (raf) edge (mek) (mek) edge (erk)
      (erk) edge (akt) (erk) edge (p38) (erk) edge[bend left=50] (pip2);
    \end{tikzpicture}
    }\\
  \small Confounding in JCI-FCI
  \end{tabular}%
  \begin{tabular}[b]{c}
    \scalebox{0.55}{
    \begin{tikzpicture}
      \node [green(ryb)] (pkc)  {\textbf{Z (PKC)}};
      \node [below left = 1cm and 2cm of pkc] (plcg) {Plc${}_{\gamma}$};
      \node [below = 1cm of plcg] (pip2) {PIP2};
      \node [below right = 0.2cm and 0.3cm of pip2] (pip3) {PIP3};
      \node [below = 1cm of pkc] (pka) {PKA};
      \node [right = 1cm of pka] (raf) {Raf};
      \node [below = 1cm of raf] (mek) {Mek};
      \node [below = 1cm of mek] (erk) {Erk};
      \node [below = 1cm of pka] (p38) {P38};
      \node [below = 2.7cm of pka] (akt) {AKT};
      \node [below left = 1.5cm and 0.5cm of pkc] (jnk) {Jnk};

      \path[-{latex[]}, line width=1.2pt, dotted, orange]
      (pka) edge (pkc)
      (jnk) edge (pkc)
      (pip3) edge (pkc)
      (raf) edge[bend right=15] (pkc)
      (p38) edge[bend right=30] (pkc)
      (plcg) edge[bend left=20] (pkc)
      (akt) edge[bend left=40] (pkc)
      (mek) edge[bend left=25] (pkc)
      (erk) edge (pkc);

      \path[-{latex[]}, line width=1.2pt, green(ryb)]
      (pip2) edge[bend left=10] (pkc);
      
      \path[-{latex[]}, dotted, thin]
      (plcg) edge (pip2) (plcg) edge (pka)
      (pip3) edge (jnk) (pip3) edge (pka) (pip3) edge (p38) (pip3) edge (mek) (pip3) edge (akt)
      (pka) edge[bend left=10] (erk) (raf) edge (mek) (mek) edge (erk)
      (erk) edge (akt) (erk) edge (p38) (erk) edge[bend left=50] (pip2);
    \end{tikzpicture}
    }\\
  \small Confounding in LS-Contexts
  \end{tabular}%
  \vspace{1em}
  \begin{tabular}[b]{c}
    \scalebox{0.55}{
    \begin{tikzpicture}
      \node [green(ryb)] (pkc) {\textbf{S (PKC)}};
      \node [below left = 1cm and 2cm of pkc] (plcg) {Plc${}_{\gamma}$};
      \node [below = 1cm of plcg] (pip2) {PIP2};
      \node [below right = 0.2cm and 0.3cm of pip2] (pip3) {PIP3};
      \node [below = 1cm of pkc] (pka) {PKA};
      \node [right = 1cm of pka] (raf) {Raf};
      \node [below = 1cm of raf] (mek) {Mek};
      \node [below = 1cm of mek] (erk) {Erk};
      \node [below = 1cm of pka] (p38) {P38};
      \node [below = 2.7cm of pka] (akt) {AKT};
      \node [below left = 1.5cm and 0.5cm of pkc] (jnk) {Jnk};

      \path[-{latex[]}, line width=1.2pt, dotted, orange]
      (pkc) edge (pka)
      (pkc) edge (raf)
      (pkc) edge[bend left=20] (mek)
      (pkc) edge[bend right=25] (p38);
      (pkc) edge[bend right=25] (pip3);

      \path[-{latex[]}, dotted, thin]
      (erk) edge (p38) (pip2) edge[bend left=20] (pkc)
      (pip3) edge (pka) (pip3) edge (p38) (pip3) edge (mek)
      (pka) edge[bend left=20] (erk)
      (plcg) edge (pip2) (plcg) edge (pka) (pip3) edge (jnk) (pip3) edge (akt)
      (raf) edge (mek) (mek) edge (erk) (erk) edge (akt)
      (erk) edge[bend left=50] (pip2);
    \end{tikzpicture}
    }\\
  \small Selection Bias in CoCo
  \end{tabular}
  \begin{tabular}[b]{c}
    \scalebox{0.55}{
    \begin{tikzpicture}
      \node [green(ryb)] (pkc) {\textbf{S (PKC)}};
      \node [below left = 1cm and 2cm of pkc] (plcg) {Plc${}_{\gamma}$};
      \node [below = 1cm of plcg] (pip2) {PIP2};
      \node [below right = 0.2cm and 0.3cm of pip2] (pip3) {PIP3};
      \node [below = 1cm of pkc] (pka) {PKA};
      \node [right = 1cm of pka] (raf) {Raf};
      \node [below = 1cm of raf] (mek) {Mek};
      \node [below = 1cm of mek] (erk) {Erk};
      \node [below = 1cm of pka] (p38) {P38};
      \node [below = 2.7cm of pka] (akt) {AKT};
      \node [below left = 1.5cm and 0.5cm of pkc] (jnk) {Jnk};

      \path[-{latex[]}, line width=1.2pt, dotted, orange]
      (pkc) edge (erk)
      (pkc) edge[bend right=25] (akt);

      \path[-{latex[]}, dotted, thin]
      (erk) edge (p38) (pip2) edge[bend left=20] (pkc)
      (pip3) edge (pka) (pip3) edge (p38) (pip3) edge (mek)
      (pka) edge[bend left=20] (erk)
      (plcg) edge (pip2) (plcg) edge (pka) (pip3) edge (jnk) (pip3) edge (akt)
      (raf) edge (mek) (mek) edge (erk) (erk) edge (akt)
      (erk) edge[bend left=50] (pip2);
    \end{tikzpicture}
    }\\
  \small Selection bias in JCI-FCI
  \end{tabular}
  \begin{tabular}[b]{c}
    \scalebox{0.55}{
    \begin{tikzpicture}
      \node [orange] (pkc) {\textbf{S (PKC)}};
      \node [below left = 1cm and 2cm of pkc] (plcg) {Plc${}_{\gamma}$};
      \node [below = 1cm of plcg] (pip2) {PIP2};
      \node [below right = 0.2cm and 0.3cm of pip2] (pip3) {PIP3};
      \node [below = 1cm of pkc] (pka) {PKA};
      \node [right = 1cm of pka] (raf) {Raf};
      \node [below = 1cm of raf] (mek) {Mek};
      \node [below = 1cm of mek] (erk) {Erk};
      \node [below = 1cm of pka] (p38) {P38};
      \node [below = 2.7cm of pka] (akt) {AKT};
      \node [below left = 1.5cm and 0.5cm of pkc] (jnk) {Jnk};

      \path[-{latex[]}, line width=1.2pt, dotted, orange]
      (jnk) edge (pkc)
      (p38) edge[bend right=30] (pkc)
      (akt) edge[bend left=40] (pkc);
      
    \path[-{latex[]}, line width=1.2pt, green(ryb)]
    (plcg) edge[bend left=20] (pkc)
    (pip3) edge (pkc)
    (raf) edge[bend right=15] (pkc)
    (mek) edge[bend left=25] (pkc)
    (erk) edge (pkc)
          (pka) edge (pkc)
    (pip2) edge[bend left=20] (pkc);

      \path[-{latex[]}, dotted, thin]
      (erk) edge (p38) 
      (pip3) edge (pka) (pip3) edge (p38) (pip3) edge (mek)
      (pka) edge[bend left=20] (erk)
      (plcg) edge (pip2) (plcg) edge (pka) (pip3) edge (jnk) (pip3) edge (akt)
      (raf) edge (mek) (mek) edge (erk) (erk) edge (akt)
      (erk) edge[bend left=50] (pip2);
    \end{tikzpicture}
    }\\
  \small Selection Bias in LS-Contexts
  \end{tabular}
  \vspace{1em}
  \caption{Retrieved latent structures across baselines.}
  \label{fig:sachs-baselines}
\end{figure}